\documentclass{article}

\usepackage{microtype}
\usepackage{graphicx}
\usepackage{subcaption}
\usepackage{booktabs} 

\usepackage{hyperref}

\usepackage[accepted]{icml2026}

\usepackage{amsmath}
\usepackage{amssymb}
\usepackage{mathtools}
\usepackage{amsthm}
\usepackage{enumitem}
\usepackage{thm-restate}
\usepackage{wrapfig,adjustbox}
\usepackage{multirow}
\usepackage[most]{tcolorbox}
\usepackage{tikz}
\usetikzlibrary{arrows.meta,positioning,matrix,shapes,calc}
\usetikzlibrary{positioning}
\usepackage{listings}
\usepackage{tabularx}
\usepackage{stmaryrd}
\newcommand{\sem}[1]{\llbracket #1 \rrbracket}
\definecolor{navy}{RGB}{10,36,99}

\newtcolorbox{promptbox}[1][]{
  enhanced, breakable, colback=white, colframe=black!15,
  boxrule=0.6pt, arc=2mm, left=6pt, right=6pt, top=6pt, bottom=6pt,
  title={\textbf{LLM Prompt}}, fonttitle=\bfseries, #1
}

\makeatletter
\newcommand{\linelabel}[1]{%
  \begingroup
  \edef\@currentlabel{\arabic{ALC@line}}%
  \label{#1}%
  \endgroup
}
\makeatother

\newcommand{\E}{\mathbb{E}}
\newcommand{\err}{\mathrm{err}}
\newcommand{\VSTAT}{\mathrm{VSTAT}}
\newcommand{\SDA}{\mathrm{SDA}}
\newcommand{\Unif}{\mathrm{Unif}}

\newcommand{\tr}{\mathrm{tr}}
\newcommand{\val}{\mathrm{val}}

\newcommand{\pv}{LLM-PV}
\newcommand{\alln}{\textbf{100}\rlap{\textbf{$^{\dagger}$}}}

\usepackage[capitalize,noabbrev]{cleveref}

\theoremstyle{plain}
\newtheorem{theorem}{Theorem}[section]

\newtheorem{lemma}[theorem]{Lemma}

\theoremstyle{definition}
\newtheorem{definition}[theorem]{Definition}

\theoremstyle{remark}
\newtheorem{remark}[theorem]{Remark}

\usepackage[textsize=tiny]{todonotes}

\icmltitlerunning{LLM Priors for ERM over Programs}

\AtBeginDocument{%
  \makeatletter
  \let\matrix\pgfmatrix
  
  \makeatother
}

\begin{document}

\twocolumn[
  \icmltitle{LLM Priors for ERM over Programs}

  \icmlsetsymbol{equal}{*}

  \begin{icmlauthorlist}
    \icmlauthor{Shivam Singhal}{equal,tamu}
    \icmlauthor{Priyadarsi Mishra}{equal,tamu}
    \icmlauthor{Eran Malach}{harvard}
    \icmlauthor{Tomer Galanti}{tamu}
  \end{icmlauthorlist}

  \icmlaffiliation{tamu}{Department of Computer Science \& Engineering, Texas A\&M University, TX, USA}
\icmlaffiliation{harvard}{Kempner Institute, Harvard University, MA, USA}

  \icmlcorrespondingauthor{Tomer Galanti}{galanti@tamu.edu}

  \icmlkeywords{Machine Learning, ICML}

  \vskip 0.3in
]



\printAffiliationsAndNotice{}  

\begin{abstract}
We study program-learning methods that are efficient in both samples and computation. Classical learning theory suggests that when the target admits a short program description, for example a short piece of ``Python code'', it can be learned from few examples by ERM over the program class. However, this approach relies on enumerating candidate programs, which is typically exponential in the description length; gradient-based training avoids this explicit search but, for some families of short programs, can require exponentially many samples to succeed. We propose \textsc{LLM-PV}, a propose-and-verify recipe that enables ERM-style selection over a discrete program class without exhaustive enumeration: a pretrained LLM induces a proposal distribution over candidate programs, each proposal is executed and scored on a held-out validation set, and the best program is selected, with no gradient updates or validation feedback used to adapt the sampling distribution. Across algorithmic tasks including parity variants, pattern matching, and primality testing, \textsc{LLM-PV} often recovers the exact underlying rule from a small labeled set and generalizes far beyond the training sequence lengths, while SGD-trained transformers, fine-tuning, in-context learning, and classical ML baselines can fit the training data yet fail to generalize reliably. Together, these results suggest that pretrained LLM priors can serve as effective search biases for ERM, narrowing the gap between statistical and computational efficiency.
\end{abstract}

\section{Introduction}

\begin{figure}[t]
\centering

\begin{adjustbox}{max width=\linewidth,center}
\begin{tikzpicture}[
  font=\normalsize,
  >=Latex,
  coltitle/.style={font=\bfseries, align=center},
  tag/.style={draw, rounded corners=3pt, inner sep=3pt, font=\scriptsize,
              text width=2.05cm, align=center, minimum height=10mm},
  tagAnaly/.style={tag, fill=blue!4,   draw=blue!70!black},
  tagCand/.style={tag, fill=cyan!10,   draw=cyan!60!black},
  tagRefn/.style={tag, fill=orange!12, draw=orange!70!black},
  tagHypo/.style={tag, fill=purple!10, draw=purple!60!black},
  tagVeri/.style={tag, fill=teal!10,   draw=teal!70!black},
  stepnode/.style={circle, draw, thick, fill=black!5, minimum size=6.8mm, inner sep=0pt},
  rule/.style={font=\scriptsize, draw, rounded corners=6pt, fill=gray!08, align=left,
               text width=4.9cm, inner sep=4pt, minimum height=10mm},
  why/.style={font=\scriptsize, draw, rounded corners=6pt, fill=blue!06, align=left,
              text width=4.9cm, inner sep=4pt, minimum height=10mm},
  link/.style={-Latex, line width=0.9pt},
  flow/.style={-Latex, dashed, gray!70, line width=0.7pt}
]
\matrix (M) [
  row sep={13.5mm, between origins},
  column sep=5.0mm,
  row 1/.style={nodes={yshift=-3mm}},
]{
  \node[coltitle] {Step};  &
  \node[coltitle] {Type};  &
  \node[coltitle] {Candidate function}; &
  \node[coltitle] {Rationale}; \\


  \node[stepnode] (s1) {1}; &
  \node (t1) {%
    \tikz[baseline=(topAB.base)]{
      \node[tagCand, minimum height=5.4mm, inner sep=2.6pt] (topAB) {Candidate A};
      \node[tagCand, below=1.6pt of topAB, minimum height=5.4mm, inner sep=2.6pt] (botAB) {Candidate B};
    }%
  }; &
  \node[rule] (r1) {{\bf Parity rule} {\color{blue} (Train acc: 50\%)}\\
                    {\bf Digit-sum parity} {\color{blue} (Train acc: 49\%)}}; &
  \node[why] (w1) {No signal; try digit-sum thresholds.}; \\

  \node[stepnode] (s2) {2}; &
  \node[tagRefn] (t2) {Refinement A}; &
  \node[rule] (r2) {{\bf Digit-sum sweep}\\
                    Best over $t$: $\hat y(x)=\mathbf{1}\{\sum_i x_{i}\ge t\}$\\
                    {\color{blue} (Train acc: 54\%)}}; &
  \node[why] (w2) {Still weak; move to position-specific tests.}; \\

  \node[stepnode] (s3) {3}; &
  \node[tagCand] (t3) {Candidate C}; &
  \node[rule] (r3) {{\bf Single-position equality}\\
                    Best over $(i,v)$: $\hat y(x)=\mathbf{1}\{x_{i}=v\}$ \\ {\color{blue} (Train acc: 60\%)}.}; &
  \node[why] (w3) {Generalize to per-position thresholds.}; \\

  \node[stepnode] (s4) {4}; &
  \node[tagCand] (t4) {Candidate D}; &
  \node[rule] (r4) {{\bf Single-position threshold}\\
                    Best over $t$: $\hat y(x)=\mathbf{1}\{x_{0}\ge t\}$ \\{\color{blue} (Train acc: 67\%)}.}; &
  \node[why] (w4) {Try an XOR fit over digit-derived bits.}; \\

  \node[stepnode] (s5) {5}; &
  \node[tagAnaly] (t5) {Analysis}; &
  \node[rule] (r5) {{\bf $\mathbb{F}_2$ fit on digit parities}\\
                    Solve $Aw=y  \pmod2$; inconsistent.}; &
  \node[why] (w5) {Try coarser summaries (half-sums).}; \\

  \node[stepnode] (s6) {6}; &
  \node[tagCand] (t6) {Candidate E}; &
\node[rule] (r6) {{\bf Half-sum features}\\
Let $S_1=\sum_{i=1}^{50}x_i$, $S_2=\sum_{i=51}^{100}x_i$.\\
$\hat y=\mathbf{1}\{S_1>S_2\}$ {\color{blue}(Train acc: 45\%)}\\
$\hat y=\mathbf{1}\{(S_1-S_2)\ \text{odd}\}$ {\color{blue}(Train acc: 49\%)}.}; &

  \node[why] (w6) {Weak signal; pivot to modular structure.}; \\

  \node[stepnode] (s7) {7}; &
  \node[tagCand] (t7) {Candidate F}; &
  \node[rule] (r7) {{\bf Mod-$3$ test} \\ $\hat y(x)=\mathbf{1}\{x\bmod 3\neq 0\}$ {\color{blue} (Train acc: 75\%)}}; &
  \node[why] (w7) {Modular signal; look for a sharper modulus-based rule.}; \\

  \node[stepnode] (s8) {8}; &
  \node[tagAnaly] (t8) {Analysis}; &
  \node[rule] (r8) {{\bf Residue classifier}\\
                    Best $m\in[2,300]$ majority by $x\bmod m$ {\color{blue} (Train acc: 98\%)}.}; &
  \node[why] (w8) {Points to a sieve-like decision rule.}; \\

  \node[stepnode] (s9) {9}; &
  \node[tagRefn] (t9) {Refinement B}; &
  \node[rule] (r9) {{\bf Divisibility probe} \\ $\hat y(x)=\mathbf{1}\{\neg(3|x)\wedge\neg(7|x)\}$ \\ {\color{blue} (Train acc: 84\%)}}; &
  \node[why] (w9) {Test more divisibility rules.}; \\

  \node[stepnode] (s10) {10}; &
  \node (t10) {%
    \tikz[baseline=(topHV.base)]{
      \node[tagHypo, minimum height=5.4mm, inner sep=2.6pt] (topHV) {Hypothesis};
      \node[tagVeri, below=1.6pt of topHV, minimum height=5.4mm, inner sep=2.6pt] (botHV) {Verification};
    }%
  }; &
  \node[rule] (r10) {{\bf Miller--Rabin primality test} {\color{blue} (Train acc: 100\%)}}; &
  \node[why] (w10) {Hypothesis confirmed.}; \\
};

\foreach \i in {1,...,10}{
  \draw[link] (w\i.west) -- (r\i.east);
}
\foreach \i [evaluate=\i as \j using int(\i+1)] in {1,...,9}{
  \draw[flow] (s\i.south) -- (s\j.north);
}
\end{tikzpicture}
\end{adjustbox}

\caption{\textbf{Reasoning trace for learning primality.}
The search starts from simple digit-level heuristics and progressively shifts toward modular structure: after parity and digit-sum baselines fail, residue-class predictors achieve high training accuracy, suggesting a sieve-like mechanism. A direct Miller--Rabin primality test then verifies the hypothesis and matches the labels perfectly.}
\label{fig:trace_prime}
\end{figure}

At its core, supervised learning asks for an algorithmic recipe for recovering structure from data:
given labeled examples $(x,y)$ from an unknown rule, produce a predictor that generalizes to unseen inputs.
A central lesson of classical learning theory is that \emph{simplicity} enables sample efficiency.
In particular, for a finite hypothesis class $\mathcal{H}$, empirical risk minimization (ERM) needs only
$\tilde O(\log|\mathcal{H}|)$ examples to generalize~\citep{10.1145/1968.1972,Vapnik1998}.
This immediately suggests a compelling perspective on \emph{program learning}:
if the target rule can be implemented by a short program of length $L$ over a token alphabet $\Sigma$,
then the effective hypothesis class size is at most $|\Sigma|^L$, and hence $\tilde O(L\log|\Sigma|)$ labeled examples suffice.

The obstacle is computation.
A direct implementation of ERM is length-first program enumeration: search all programs of length at most $L$ and return the first that fits the data.
But the number of candidates grows exponentially, $|\mathcal{L}_{\le L}|=\sum_{\ell=1}^{L}|\Sigma|^\ell=\Theta(|\Sigma|^L)$,
so even when verification is fast, the overall runtime scales as $\Theta(m|\Sigma|^L)$.
Thus, the classical route to sample-efficient learning of short programs is computationally infeasible for even moderate $L$
(e.g., $L{=}20$ and $|\Sigma|{=}10$ already yields $10^{20}$ candidates).

\textbf{Modern deep learning flips this trade-off.\enspace}
Instead of searching explicitly over programs, one typically trains high-capacity predictors by stochastic gradient descent (SGD)~\citep{robbins1951stochastic,Bottou2010LargeScale}.
This is computationally attractive: the optimization cost is often nearly linear in the data, and large models can fit complex training sets.
Yet computational tractability does not imply statistical efficiency on structured rule families.
In the statistical query (SQ) view~\citep{10.1145/293347.293351}, broad classes of gradient-based procedures can require exponentially many samples
on high--SQ-dimension targets such as parity- and cryptographic-like functions, despite the fact that these targets admit succinct program descriptions.
In other words, the difficulty is not that the rule is long or unstructured, but that the learning algorithm accesses it through a bottlenecked interface
(gradients, finite precision, local updates) that is poorly matched to discrete program structure.

\begin{tcolorbox}[colback=blue!5!white, colframe=blue!30!black, arc=1pt]
\centering
\textbf{{\em Can we recover the sample efficiency of ERM over short programs without paying the full cost of exhaustive enumeration?}}
\end{tcolorbox}

{\bf Contributions.\enspace}
We focus on a small-data regime where the target has a short discrete description, so finite-class ERM enjoys strong classical generalization guarantees, yet the two standard computational routes are unsatisfying: length-first enumeration is exponential in program length, while gradient-based training can be sample-inefficient on structured, SQ-hard families.
Our goal is not to introduce new PAC/SQ machinery or new constraint solvers. Instead, we make ERM over discrete program classes more computationally accessible by using pretrained LLMs only as a proposal mechanism for search, while keeping selection grounded in executable verification and held-out error. Concretely, we ask whether a pretrained LLM can serve as a useful \emph{search prior} that reduces the need for exhaustive enumeration without changing the learning objective.
Our contributions are:

\begin{itemize}[leftmargin=10pt, itemsep=0.35ex, topsep=0.25ex, parsep=0pt, partopsep=0pt]

\item \textbf{A clean compute--sample tension for short programs.\enspace}
Using classic SQ/PAC tools, we revisit a standard realizable family (planted $k$-parity) to formalize a regime that is central to small-data program learning.
Finite-class ERM over length-$L$ programs achieves the usual sample bound $m=\tilde O(L\log|\Sigma|)$ (Eq.~\ref{eq:pac}), but its runtime is exponential in $L$.
In contrast, finite-precision mini-batch \emph{coordinate} SGD may be computationally convenient per step, yet requires many fresh examples to reach nontrivial error on this family, namely $q=TB=\Omega(\sqrt{\binom{n}{k}}/2^b)$ (Prop.~\ref{prop:sgd-parity-lb-restated}). Together, these results isolate a concrete small-data setting where both naive enumeration and gradient-based training break down, for complementary reasons.

\item \textbf{ERM-style selection with an LLM proposal prior.\enspace} Inspired by classic work on \emph{programming by example} (PBE) and the growing literature on LLM-guided program synthesis, we introduce LLM-PV. LLM-PV is a propose-and-verify procedure that cleanly separates roles: (i) a pretrained LLM induces a data-dependent proposal distribution over candidate programs (or edits), (ii) an external verifier compiles and executes candidates, and (iii) we select by held-out validation error (Alg.~\ref{alg:llm-ktry}). This makes the role of pretrained reasoning precise: it serves as a \emph{search bias} that concentrates trials on plausible hypotheses, while the learning objective and selection rule remain similar to ERM. 

\item \textbf{Sample-efficient rule recovery and length generalization across algorithmic tasks.\enspace}
Across algorithmic tasks including parity variants, pattern matching, palindromes, Dyck-2, primality variants and pseudo-random functions, \textsc{LLM-PV} often recovers the underlying rule from 200 labeled examples
(Tabs.~\ref{tab:all_methods_n100} and~\ref{tab:pv_vs_all}),
and frequently outputs compact, input-length-invariant programs that generalize far beyond the training length
(Fig.~\ref{fig:length_generalize_ablation}). In the same settings, classic ML methods (e.g., SVM, XGBoost), SGD-trained transformers and standard adaptation baselines (fine-tuning and in-context learning) commonly fit the training set but fail to generalize reliably as input length or dimension grows
(Tab.~\ref{tab:pv_vs_all}, Fig.~\ref{fig:size_ablation}).

\item \textbf{Auditability of both the learned hypothesis and the search process.\enspace} The output is executable, human-readable code, and the full learning trajectory is inspectable:
we log the sequence of proposed programs/edits, execution outcomes, and verification diagnostics, yielding an auditable propose--verify trace (Fig.~\ref{fig:trace_prime} and Fig.~\ref{fig:trace_suite} in App.~\ref{app:traces}). This enables direct debugging of failure modes and mechanistic validation of successes, beyond reporting final accuracy.

\end{itemize}




\section{Related Work}

{\bf PAC learning, Occam's razor and short programs.\enspace}
We follow the the classical generalization theory, where finite-class ERM has sample complexity \(O(\log|\mathcal H|)\)~\citep{10.1145/1968.1972,vapnik1971uniform,Vapnik1998}.
The ``short program'' view instantiates Occam/MDL: a hypothesis encodable in \(L\) symbols over alphabet \(\Sigma\) admits bounds of order \(O(L\log|\Sigma|)\), up to confidence terms~\citep{blumer1987occams,86996,720554,10.5555/534258,10.1145/279943.279989}. The length-first search (Alg.~\ref{alg:program-enum}) realizes this ERM guarantee but incurs exponential time in description length, reflecting the classic universal-search trade-off~\citep{levin1973universal,solomonoff1964formal1,solomonoff1964formal2}.

{\bf Statistical query (SQ) learning and hardness of learning.\enspace} The SQ framework and its refinements~\citep{10.1145/293347.293351,10.1145/195058.195147,pmlr-v65-feldman17c,reyzin2020statisticalqueriesstatisticalalgorithms} yield lower bounds for many concept classes. Parity and related families have large SQ dimension under the uniform distribution, so any SQ learner needs exponentially many (tolerant) queries to achieve nontrivial correlation~\citep{10.1145/195058.195147,10.1145/3046674,klivans2007unconditional,klivans_et_al:LIPIcs.APPROX-RANDOM.2014.793,giapitzakis2025statisticalquerycomplexitylearning}. Intuitively, mini-batch SGD is itself approximately an SQ algorithm: each update averages a bounded statistic over samples~\citep{10.1145/3046674,doi:10.1137/16M1078471,NEURIPS2021_cc225865,barak2022hidden}. Hence, SQ lower bounds transfer directly to SGD, making its iteration complexity grow with the SQ dimension—exponentially for parities and pseudorandom families under the uniform distribution. Our analysis formalizes this connection, showing how SQ hardness induces exponential sample requirements for gradient-based methods.

{\bf Gradient-based training on algorithmic reasoning.\enspace} This SQ perspective aligns with extensive empirical evidence: even when a neural family can represent the target compactly, SGD-trained networks often fail on parity-like and compositional algorithmic tasks without strong inductive bias or large data~\citep{10.5555/3305890.3305998,pmlr-v80-safran18a,10.5555/3495724.3497433,barak2022hidden}. More broadly, work on neural \emph{trainability} separates representational power from optimization and sample efficiency~\citep{NEURIPS2019_5481b2f3,10.5555/3294996.3295003}, and phenomena such as ``grokking''—delayed generalization after long training—highlight a gap between statistical optima and what SGD finds in practice~\citep{PowerEtAl2022Grokking}. These observations motivate alternatives that retain finite-class guarantees while improving practical search efficiency.

{\bf LLM-guided search and iterative improvement.\enspace} Program synthesis and discrete search over programs has been studied for decades, and it has long been viewed as a difficult problem because the hypothesis space is combinatorial and brittle to small specification changes.
The LLM era introduced a new primitive: strong \emph{language-conditioned priors} over discrete objects, which can be used as proposals in a search loop. A prominent thread uses iterative propose--critique--revise driven by natural-language feedback (e.g., Self-Refine, Reflexion), and related approaches treat feedback as an optimization signal, including so-called textual gradients
\citep{10.5555/3666122.3668141,shinn2023reflexion,yuksekgonul2024textgradautomaticdifferentiationtext}.
In parallel, evolutionary and neuro-symbolic frameworks use LLMs to propose edits or modular building blocks that are refined via mutation and selection
\citep{pourcel2025selfimproving,novikov2025alphaevolvecodingagentscientific,AlphaEvolve2025,bhansali2024legolanguagemodelbuilding}.

{\bf Programming by examples and LLM-assisted program synthesis.\enspace} A complementary tradition studies \emph{programming by examples} (PBE) and inductive program synthesis, where the input is a finite set of input--output examples or tests and the goal is to output a program consistent with them
\citep{10.1145/1925844.1926423,10.1145/2666356.2594333,10.1007/978-3-642-31424-7_44,6679385,gulwani2017program}.
This literature emphasizes search over discrete program spaces (often within a DSL), including enumerative synthesis, version-space algebras, and CEGIS-style loops
\citep{10.1023/A:1025671410623,10.1145/1168919.1168907,polikarpova2016program,solarlezama2013program},
often complemented by ranking, inductive bias, or user interaction to resolve ambiguity
\citep{10.1145/1925844.1926423,gulwani2017program,6679385}.

Recent work revisits PBE in the LLM era. One early paradigm treats the LLM itself as the ``program'' via \emph{in-context learning} (ICL): demonstrations, often with chain-of-thought, induce task procedures without parameter updates \citep{NEURIPS2020_1457c0d6,min-etal-2022-rethinking,wei2022chain}, and this phenomenon has been analyzed theoretically \citep{pmlr-v202-von-oswald23a,aky,shen2024do,dewynter2025incontextlearninglearning}. Empirically, however, pretrained LLMs can be brittle when asked to synthesize \emph{explicit} programs from examples alone, with performance depending strongly on distributional match and fine-tuning \citep{NEURIPS2024_4eff61b7}. Other work proposes LLM-guided synthesis pipelines that retain execution-based verification while improving robustness via decomposition and compositional search \citep{khan2025llmguidedcompositionalprogramsynthesis}. Complementary benchmarks study \emph{interactive} settings, where an agent can query a hidden target function and refine solutions using feedback or counterexamples \citep{wei2025codearc}. More broadly, LLMs have been combined with program analysis and synthesis components to improve the reliability of generated code in practical settings \citep{10.1145/3510003.3510203}.

{\bf Our lens.\enspace} We adopt the propose--verify view from program synthesis and PBE and repurpose it for \emph{learning from i.i.d.\ data} with an explicit generalization objective. We treat the hypothesis space as a discrete class of executable programs and use a pretrained LLM only as an inductive-bias mechanism: a proposal distribution over candidate programs. Each proposal is verified by execution on the observed samples, and the final program is chosen by validation-based model selection, so the criterion is held-out error rather than mere consistency with a finite specification. In this sense, we inherit the verification discipline of synthesis, but replace interactive counterexamples and hand-crafted templates with distributional learning and ERM-style selection from $S$.

\section{Theoretical Analysis}

\subsection{Problem Setup} 

We study \emph{inductive program synthesis} (``program learning''): the target is a binary function \(y:\mathcal X \to \{\pm 1\}\) implemented by a short program in a fixed language, and the learner receives i.i.d.\ examples \(S=\{(x_i,y(x_i))\}_{i=1}^m\) with \(x_i\overset{\mathrm{i.i.d.}}{\sim}D\).
Throughout, we assume the \emph{realizable} setting, i.e., \(y\in \mathcal{L}\), where \(\mathcal{L}\) is the class of total functions computed by programs in the language (formalized below).

{\bf Language and semantics.\enspace} Fix a finite alphabet \(\Sigma\) and a programming language \(\mathcal L\subseteq \Sigma^\ast\).
Each string \(u\in\mathcal L\) has semantics \(\llbracket u\rrbracket:\mathcal X\rightharpoonup\{\pm 1\}\), a (possibly partial) function that may fail to compile or fail to halt.
We write \(\llbracket u\rrbracket(x)=\bot\) when \(u\) does not produce an output on \(x\).
Let $\mathcal C \;:=\; \bigl\{\, f:\mathcal X\to\{\pm 1\}\ :\ \exists\,u\in\mathcal L~~ \text{s.t. }~~ \llbracket u\rrbracket \text{ is total and } \llbracket u\rrbracket=f \,\bigr\}$. We denote the length of \(u\) by \(|u|\) (in symbols over \(\Sigma\)) and write \(\mathcal L_\ell:=\{u\in\mathcal L:\ |u|=\ell\}\). A program is considered total if it defines an output for every input—i.e., it never fails to compile and halts on all $x \in \mathcal{X}$, returning a  $\pm 1$ label.

{\bf Objective.\enspace} For a hypothesis \(h:\mathcal X\to\{\pm 1\}\), define population error
\(\mathrm{err}_{D}(h):=\Pr_{x\sim D}[h(x)\neq y(x)]\)
and empirical error
\(\mathrm{err}_{S}(h):=\tfrac{1}{m}\sum_{i=1}^m \mathbf 1\{h(x_i)\neq y_i\}\).
The goal is to output a program \(u\in\mathcal L\) whose total semantics \(\llbracket u\rrbracket\) attains small \(\mathrm{err}_D\).

{\bf Computational model.\enspace} When executing a candidate program \(u\) on input \(x\), we allow a time budget \(T\in\mathbb N\) per call; if \(u\) fails to compile or does not halt within time \(T\), we treat the outcome as \(\bot\) and reject \(u\) as a hypothesis.
This makes search procedures well-defined even when \(\llbracket u\rrbracket\) is partial.

{\bf Short-program regime.\enspace}
We will frequently analyze the \emph{short-program} subclass $\mathcal H_\ell := \bigl\{\llbracket u\rrbracket\ :\ u\in\mathcal L_\ell,\ \llbracket u\rrbracket \text{ total}\bigr\}$, where $\mathcal H = \bigcup_{\ell\ge 1}\mathcal H_\ell$ and compare (i) explicit search over \(\mathcal H_\ell\) (finite-class ERM) to (ii) gradient-based learners \(h_\theta\) drawn from a proxy hypothesis family \(\{h_\theta:\theta\in\Theta\}\).

\subsection{Program Enumeration}

One approach to program learning is \emph{program enumeration}. Given a language $\mathcal{L}$ with token alphabet $\Sigma$, enumerate candidate programs in a canonical order (e.g., by length and then lexicographically), and return the first program that is consistent with the training sample (Alg.~\ref{alg:program-enum} in App.~\ref{app:proof_pac}).

Despite its simplicity, enumeration is well known to be highly sample efficient. By standard PAC bounds for finite hypothesis classes~\citep{10.1145/1968.1972} (see also Cor.~2.3 of~\citep{10.5555/2621980}), if the target function can be implemented by some length-$L$ program in $\mathcal{L}$, then with probability at least $1-\delta$ over the selection of $S$, the returned program $h$ satisfies
\begin{equation}\label{eq:pac}
\err_D(h)\;\le\;m^{-1}[L\log|\Sigma|+\log(2L^2/\delta)].
\end{equation}
Thus, the required sample size scales with the program length of the target function. For completeness, we prove this inequality in App.~\ref{app:proof_pac}.

The main limitation is computational. In the worst case, length-first enumeration must examine essentially all programs up to length $L$, which is exponential in $L$: $\Omega(|\Sigma|^{L})$. For example, with an ASCII-sized alphabet $|\Sigma|=128$ and $L=10$, this corresponds to $128^{10}\approx 1.2\times 10^{21}$ candidate strings, making brute-force search infeasible in practice.

\subsection{Gradient-Based Optimization}

A alternative method is to search for a predictor in a large parametric class (e.g., neural networks) using gradient-based optimization. To reason about the information such methods obtain from data, we use the statistical-oracles viewpoint: many iterative optimization procedures can be modeled as algorithms that access the data distribution only through estimates of expectations of bounded functions and closely related sample-based variants~\citep{10.1145/293347.293351,doi:10.1137/16M1078471,10.1145/3046674,reyzin2020statisticalqueriesstatisticalalgorithms}.

{\bf Finite-precision SGD.\enspace} There are two standard ways to connect SGD to statistical-oracles: either inject independent noise so each update becomes a tolerant expectation query, or focus on finite-precision coordinate mini-batch SGD where each per-example coordinate contribution simulates a $1$-STAT$(b)$ query and $q=TB$ such queries can be simulated via $\VSTAT$ by the standard reduction \citep[Thm.~B.4]{doi:10.1137/16M1078471}. We therefore model mini-batch coordinate SGD as an algorithm that only accesses fresh data through \emph{finite-precision} (per-example) gradient information. At iteration $t$, the algorithm selects parameters $\theta_t$ and a coordinate $j_t\in[d]$, draws a fresh mini-batch $z_{t,1},\dots,z_{t,B}\sim D$, and updates using the $b$-bit quantized mini-batch coordinate gradient $\theta_{t+1}=\theta_t-\eta_t\,\widehat g^{(b)}_{t,j_t}\,e_{j_t}$, $\widehat g^{(b)}_{t,j_t}
~:=~
\frac{1}{B}\sum_{i=1}^B Q_b\left(\partial_{j_t}\ell(\theta_t,z_{t,i})\right)$,
where $Q_b$ is a $b$-bit quantizer and $\eta_t$ is the step size. The only interaction with fresh samples is via these $b$-bit per-example gradient values. 

{\bf Planted $k$-parity.\enspace} A canonical SQ-hard family is parity under the uniform distribution, and closely related planted parity testing problems inherit the same ``low-correlation'' structure used in SQ lower bounds \citep{10.1145/792538.792543,10.1145/293347.293351,reyzin2020statisticalqueriesstatisticalalgorithms,10.1145/3046674}. Fix $n\ge 1$ and $k\in\{1,\dots,n\}$, let $\mathcal X=\{0,1\}^n$, and let $\mathcal S_k$ be the set of
$k$-sparse vectors in $\{0,1\}^n$. For each $s\in\mathcal S_k$, define the target $f_s(x)=(-1)^{\langle s,x\rangle}$ and the realizable distribution
$D_s$ on $\mathcal X\times\{\pm1\}$ by $x\sim\Unif(\mathcal X)$ and $y=f_s(x)$.
Let $N:=|\mathcal S_k|=\binom{n}{k}$ be the number of possible targets. 

The following is a technical application of using the two-step reduction using \citep[Thm.~B.4]{doi:10.1137/16M1078471} and the classic results from SQ learning \citep{reyzin2020statisticalqueriesstatisticalalgorithms}. For completeness, we provide a full formal analysis in App.~\ref{sec:sgd-1statb}.

\begin{restatable}{proposition}{sgdlowerbound}
\label{prop:sgd-parity-lb-restated}
Consider any (possibly adaptive, randomized) procedure that runs $T$ iterations of
$b$-bit mini-batch \emph{coordinate} SGD with batch size $B$ on the above family,
meaning that on fresh examples $(x,y)\sim D_s$ its only access to the data is through
$b$-bit quantized per-example coordinate gradients as in the update displayed above.
Let the procedure output a hypothesis $h:\mathcal X\to\{\pm1\}$, and write $q:=TB$
for the total number of fresh examples used.

If for every $s\in\mathcal S_k$ the procedure achieves nontrivial population error $\err_{D_s}(h)\le \frac14$ with probability at least $\beta \geq 5/6$, then necessarily $q = \Omega\left(\frac{\sqrt{N}}{2^b}\right)$.
\end{restatable}

Program enumeration (Alg.~\ref{alg:program-enum}) and finite-precision coordinate SGD fail for planted $k$-parity for complementary reasons. Let $N=\binom{n}{k}$ be the number of $k$-parities. Enumerating programs of length $\le L$ costs $\Theta(m,|\Sigma|^{L})$. A $k$-parity has a description of length $L=\Theta(\log N)$ (ASCII/Python), and Eq.~\ref{eq:pac} gives $m=\tilde O(L)=\tilde O(\log N)$ for nontrivial error, so the resulting work is $\Theta(|\Sigma|^{\Theta(\log N)}\log N)=N^{\Theta(1)}\log N$, i.e., polynomial in $N$. By contrast, finite-precision coordinate SGD runs in $\Theta(m\cdot \mathrm{Cost}\nabla)$ time on $m$ fresh examples, but Prop.~\ref{prop:sgd-parity-lb-restated} needs $m=\Omega(\sqrt{N}/2^{b})$ to reach population error $\le 1/4$, yielding total work $\Omega((\sqrt{N}/2^{b})\cdot \mathrm{Cost}\nabla)$. In short: enumeration is expensive in \emph{search} (about $N^{\Theta(1)}$ candidates), while SGD is expensive in \emph{data} (about $\sqrt{N}$ fresh examples), and both become infeasible as $N=\binom{n}{k}$ grows (for constant $k$, $N=\Theta(n^k)$ so costs scale as $n^{\Theta(k)}$ vs.\ $n^{k/2}$, up to $\mathrm{Cost}_\nabla$).

\section{New Lens}

\begin{figure}[t]
\centering
\begin{adjustbox}{width=\linewidth}
\begin{tikzpicture}[
  font=\small,
  block/.style={draw, rounded corners=6pt, fill=navy!12, align=center, minimum height=10mm, text width=3.6cm, inner sep=4pt},
  accept/.style={draw, rounded corners=6pt, fill=green!12, align=center, minimum height=10mm, text width=3.6cm, inner sep=4pt},
  stop/.style={draw, rounded corners=6pt, fill=red!12, align=center, minimum height=10mm, text width=3.6cm, inner sep=4pt},
  line/.style={-Latex, thick},
  flow/.style={-Latex, dashed, navy!70, line width=0.9pt}
]

\node[block]  (train)  at (0,0)    {Training set\\$S_{\mathrm{tr}}$};
\node[block]  (prompt) at (4.6,0)  {Build prompt\\$\Pi(S_{\mathrm{tr}})$};
\node[block]  (llm)    at (9.2,0)  {Sample candidate\\programs from LLM\\(temperature $\tau$)};
\node[block]  (compile)at (13.8,0) {Compile \& filter:\\discard invalid or\\duplicate programs};
\node[block]  (val)    at (18.4,0) {Evaluate on $S_{\mathrm{val}}$:\\compute $\err_{S_\mathrm{val}}(h)$};
\node[accept] (select) at (23.0,0) {Update best:\\keep $h^\star$ with\\lowest validation error};
\node[stop]   (stop)   at (9.2,-2.1) {Stop:\\after $k$ rounds};

\draw[line] (train.east) -- (prompt.west);
\draw[line] (prompt.east) -- (llm.west);
\draw[line] (llm.east) -- (compile.west);
\draw[line] (compile.east) -- (val.west);
\draw[line] (val.east) -- (select.west);

\draw[flow]
  (select.north) to[out=75, in=85, looseness=0.25]
  node[midway, above, yshift=1mm]{next round}
  (llm.north);

\draw[flow] (llm.south) -- node[midway, right, xshift=2mm]{after $k$ rounds} (stop.north);

\end{tikzpicture}
\end{adjustbox}
\vspace{-0.5cm}
\caption{{\bf An illustration of LLM-PV.} A prompt built from $S_{\mathrm{tr}}$ conditions the LLM to sample candidate programs. We compile and filter candidates, evaluate validation error on $S_{\mathrm{val}}$, and maintain the best-so-far hypothesis $h^\star$. The procedure runs for $k$ rounds and returns the program with the lowest validation error among all candidates evaluated.}
\label{fig:experiment-flow-validation}
\end{figure}

\begin{algorithm}[t]
\caption{LLM-PV: $k$-try LLM propose-and-verify with validation}
\label{alg:llm-ktry}
\begin{algorithmic}[1]
\REQUIRE train/validation sets $S_{\mathrm{tr}},S_{\mathrm{val}}$; trials $k$; prompt builder $\Pi$; temperature $\tau$
\ENSURE Best program $u^\star$ with $h^\star=\sem{u^\star}$
\STATE \linelabel{alg:init} Initialize $\err^\star \leftarrow 1$, $(u^\star,h^\star) \leftarrow (\bot,\bot)$, $\mathcal U \leftarrow \emptyset$
\FOR{$t=1,\dots,k$}
  \STATE \linelabel{alg:query} Sample a candidate program $u$ from LLM$(\Pi(S_{\mathrm{tr}});\tau)$
  \IF{$u\in\mathcal U$} \STATE \textbf{continue} \ENDIF
  \STATE $\mathcal U\leftarrow\mathcal U\cup\{u\}$; compile and set $h\leftarrow\sem{u}$
  \IF{$h$ undefined on some $x\in S_{\mathrm{tr}}\cup S_{\mathrm{val}}$} \STATE \textbf{continue} \ENDIF
  \STATE \linelabel{alg:errors} Compute $\err_{S_\mathrm{val}}(h)$
  \IF{$\err_{S_\mathrm{val}}(h)<\err^\star$}
    \STATE \linelabel{alg:update} $(\err^\star,u^\star,h^\star)\leftarrow(\err_{S_\mathrm{val}}(h),u,h)$
  \ENDIF
\ENDFOR
\STATE \linelabel{alg:return} \textbf{return} $(u^\star,h^\star)$
\end{algorithmic}
\end{algorithm}

A natural alternative to brute-force program enumeration is \emph{sampling with verification}: draw candidates and return the one with smallest held-out error. But \emph{uniform} sampling is hopeless: if $u\sim \mathrm{Unif}(\mathcal L_L)$ and the target is $u^\star$, then $\Pr[u=u^\star]=|\Sigma|^{-L}$, so the expected trials are $|\Sigma|^{L}$, essentially matching enumeration. Thus the bottleneck is not verification, but the \emph{proposal distribution}.

This motivates a \emph{proposer} that, given $S_{\mathrm{tr}}$, induces a data-dependent distribution $q(\cdot\mid S_{\mathrm{tr}})$ placing non-negligible mass on programs consistent with the observed pairs. Given such a proposal, the rest is standard: compile/execute candidates, discard invalid programs, and select by validation error.

A strong pretrained LLM provides a natural proposal bias: prompted with $S_{\mathrm{tr}}$, it generates code concentrated on recognizable templates and concise rules suggested by the examples, rather than spreading mass nearly uniformly over the vast program space. In this view, the LLM is used as a data-dependent \emph{proposal prior distribution}.

Alg.~\ref{alg:llm-ktry} instantiates this propose--verify--select pipeline.
Based on Thm.~11.1 in~\citep{10.5555/2621980}, the guarantee depends on the \emph{proposal mass}
$p_\epsilon(S_{\mathrm{tr}}):=\Pr_{h\sim q(\cdot\mid S_{\mathrm{tr}})}[\err_D(h)\le \epsilon]$,
plus the usual finite-class validation penalty.

\begin{restatable}{proposition}{pv}
\label{thm:pv}
Fix $\epsilon\ge 0$ and $\delta\in(0,1)$. Draw independent samples
$S_{\mathrm{tr}}\sim D^{m_{\mathrm{tr}}}$ and $S_{\mathrm{val}}\sim D^{m_{\mathrm{val}}}$ with labels from the target $y:\mathcal X\to\{\pm1\}$.
Run Alg.~\ref{alg:llm-ktry} for $k$ trials, and assume each trial outputs an accepted total hypothesis
$h_t:\mathcal X\to\{\pm1\}$ using only $S_{\mathrm{tr}}$. Let
$h^\star \in \arg\min_{t\in[k]} \err_{S_{\mathrm{val}}}(h_t)$ be the returned hypothesis. If $k \geq \lceil \frac{\log(\delta/2)}{\log \big(1-p_\epsilon(S_{\mathrm{tr}})\big)} \rceil$, then with probability at least $1-\delta$ over $(S_{\mathrm{tr}},S_{\mathrm{val}})$ (and all algorithmic randomness), $\err_D(h^\star) \le \epsilon + 2\sqrt{\log(4k/\delta)/(2m_{\mathrm{val}})}$.
\end{restatable}

The guarantee separates the two roles in LLM-PV. The LLM enters through
$p_\epsilon(S_{\mathrm{tr}})$, the probability that a fresh accepted proposal has population error at most $\epsilon$. Given $k$ candidates, validation selection pays the standard finite-class price $\tilde O(\sqrt{\log(k/\delta)/m_{\mathrm{val}}})$. Thus the LLM need not be trusted as a predictor; it only needs to put enough mass on low-error programs so one appears within $k$ trials, after which performance is certified by held-out evaluation.

{\bf Comparing the number of programs observed.\enspace} Uniform sampling from $\mathcal L_L$ hits any fixed program with probability $|\Sigma|^{-L}$, so it takes on the order of $|\Sigma|^{L}$ trials to observe it. In LLM-PV the relevant quantity is the proposal mass $p_\epsilon(S_{\mathrm{tr}})$ on $\epsilon$-good hypotheses: by Prop.~\ref{thm:pv} it suffices that $k \ge \lceil \frac{\log(2/\delta)}{-\log(1-p_\epsilon(S_{\mathrm{tr}}))} \rceil$, and for $p_\epsilon(S_{\mathrm{tr}})\ll 1$ this is $k \approx \frac{\log(2/\delta)}{p_\epsilon(S_{\mathrm{tr}})}$. Thus, LLM-PV replaces the exponential search size $|\Sigma|^{L}$ by an \emph{effective} search size $1/p_\epsilon(S_{\mathrm{tr}})$, improving over enumeration whenever $1/p_\epsilon(S_{\mathrm{tr}})\ll |\Sigma|^{L}$ (equivalently $p_\epsilon(S_{\mathrm{tr}})\gg |\Sigma|^{-L}$).

{\bf Interpretability.\enspace} LLM-PV makes both the \emph{learned object} and the \emph{learning process} transparent: each run returns (i) an executable program and (ii) a reasoning trace that logs the proposed candidates and their observed errors (often including train accuracy). Fig.~\ref{fig:trace_prime} shows a \texttt{GPT-5-Thinking} trace (ChatGPT web UI) on IsPrime with $m{=}100$ examples of 100-digit inputs: it starts with simple digit-level rules (parity and digit-sum tests), then shifts to modular structure, and ends with a Miller--Rabin primality test. Because each step is linked to the errors it aims to fix, failures are auditable and successes are inspectable: the final symbolic program can be unit-tested, stress-tested out of distribution, and re-run or edited.

\section{Experiments}


\begin{table*}[t]
\centering
\resizebox{\textwidth}{!}{%
\begin{tabular}{@{}l|ccccccccccccc@{}}
\toprule
\textbf{Model / Task} &
\textbf{FP} & \textbf{FHP} & \textbf{Pattern1} & \textbf{Pattern2} &
\textbf{3-P} & \textbf{10-P} & \textbf{IsPal} & \textbf{Dyck-2} &
\textbf{CAP} & \textbf{IsPrime} & \textbf{IsPrime+47} & \textbf{CycleCount} & \textbf{SHA} \\
\midrule
\texttt{XGBoost}
& 49.9 & 50.0 & 50.6 & 52.3 & 49.9 & 49.3 & 83.6 & 49.7 & 50.4 & 61.8 & 59.3 & 51.1 & 50.3 \\

\texttt{Random Forest}
& 49.8 & 49.7 & 50.5 & 52.5 & 49.7 & 49.3 & 77.8 & 50.1 & 49.2 & 58.7 & 55.7 & 50.8 & 50.6 \\

\texttt{SVM}
& 49.9 & 49.9 & 49.8 & 51.6 & 50.6 & 48.9 & 69.7 & 50.1 & 50.1 & 55.7 & 55.6 & 47.0 & 50.2 \\

\texttt{Genetic Algorithm}
& 49.8 & 52.5 & 51.1 & 49.7 & 47.0 & 54.0 & 87.1 & 51.9 & 48.6 & 51.7 & 53.5 & 49.3 & 49.0 \\

\texttt{TabPFN}~\cite{hollmann2022tabpfn}
& 50.1 & 50.5 & 49.2 & 51.2 & 49.3 & 50.5 & 49.4 & 50.5 & 48.0 & 57.1 & 65.4 & 51.3 & 50.6 \\

SGD: \texttt{Qwen3-1.7B}
& 49.3 & 50.6 & 51.5 & 53.6 & 49.7 & 50.3 & 49.2 & 51.4 & 49.4 & 58.8 & 53.6 & 53.5 & 49.8 \\

SGD: \texttt{DeepSeek-Coder-1.3B}
& 52.2 & 50.5 & 51.3 & 55.1 & 54.3 & 51.3 & 51.6 & 49.4 & 50.0 & 54.5 & 56.0 & 53.9 & 49.8 \\

SGD: \texttt{Llama-3.2-1B}
& 52.2 & 48.6 & 50.0 & 54.3 & 51.5 & 50.6 & 50.3 & 51.5 & 50.2 & 52.7 & 51.7 & 54.7 & 51.0 \\

\midrule
FT: \texttt{Llama-3.2-1B}
& 50.4 & 51.3 & 61.5 & 57.6 & 51.0 & 50.5 & 49.6 & 52.7 & 50.7 & 57.2 & 58.7 & 65.7 & 50.8 \\

FT: \texttt{Qwen3-1.7B}
& 51.2 & 50.7 & 59.6 & 59.4 & 50.9 & 50.7 & 49.6 & 52.5 & 50.8 & 58.3 & 52.9 & 57.4 & 51.2 \\

FT: \texttt{DeepSeek-Coder-1.3B}
& 50.5 & 50.2 & 56.4 & 57.7 & 50.7 & 50.4 & 49.7 & 51.7 & 50.2 & 58.9 & 60.3 & 55.7 & 49.7 \\

\midrule
ICL: \texttt{Qwen3-30B-Instruct} (no tools)
& 53.0 & 54.0 & 56.0 & 51.0 & 50.0 & 50.0 & 47.0 & 48.0 & 49.0 & 50.0 & 58.0 & 46.0 & \textbf{54.0} \\

ICL: \texttt{Qwen3-Coder-30B-Instruct} (no tools)
& 50.0 & 50.0 & 50.0 & 50.0 & 50.0 & 44.0 & 51.0 & 50.0 & 49.0 & 47.0 & 60.0 & 47.0 & 49.0 \\

ICL: \texttt{DeepSeek-Coder-33B-Instruct} (no tools)
& 43.0 & 54.0 & 54.0 & 47.0 & 47.0 & 48.0 & 51.0 & 54.0 & 44.0 & 53.0 & 61.0 & 41.0 & 52.0 \\

ICL: \texttt{Gemini-3.1-Pro} (with tools)
& 48.0 & 54.0 & 83.0 & 57.0 & 50.0 & 45.0 & \textbf{100} & 51.0 & 51.0 & 52.0 & 56.0 & 50.0 & 45.0 \\

ICL: \texttt{GPT-5-Thinking} (no tools)
& 82.0 & 66.0 & 55.0 & 41.0 & 47.0 & 47.0 & 79.0 & 57.0 & 51.0 & 55.0 & 47.0 & 43.0 & 44.0 \\

ICL: \texttt{GPT-5-Thinking} (with tools)
& 99.0 & 88.0 & 51.0 & 51.0 & 75.0 & 79.0 & 91.0 & 59.0 & 46.0 & 62.0 & 56.0 & 51.0 & 46.0 \\

\midrule
MLAgentBench~\cite{10.5555/3692070.3692884}: \texttt{GPT-5-Thinking} (with tools)
& 53.6 & 51.0 & 33.4 & 60.9 & 32.8 & 50.1 & 31.6 & 52.1 & 50.0 & 32.6 & 33.1 & 49.1 & 49.9 \\

AIDE-ML~\cite{jiang2025aideaidrivenexplorationspace}: \texttt{GPT-5-Thinking} (with tools)
& 49.5 & 49.9 & 50.3 & 50.9 & 49.5 & 49.2 & 50.0 & 50.5 & 50.2 & 55.5 & 49.7 & 50.0 & 50.8 \\

\midrule
LLM-PV: \texttt{Qwen3-30B-Instruct} (no tools)
& 50.4 & 49.6 & 51.1 & 60.9 & 50.4 & 51.6 & 49.9 & 80.4 & 50.2 & 50.6 & 50.2 & 49.9 & 49.7 \\

LLM-PV: \texttt{DeepSeek-Coder-33B-Instruct} (no tools)
& \textbf{100} & 50.9 & 50.9 & 50.5 & 51.6 & 50.8 & 51.0 & 57.3 & 50.2 & 50.0 & 50.2 & 50.0 & 51.5 \\

LLM-PV: \texttt{OLMo-3-32B-Instruct} (no tools)
& \textbf{100} & 50.5 & 50.9 & 83.5 & 50.2 & 50.2 & 52.7 & 50.5 & 50.8 & 50.8 & 50.5 & 52.7 & 50.1 \\

LLM-PV: \texttt{Gemini-3.1-Pro} (with tools)
& \textbf{100} & \textbf{100} & \textbf{100} & \textbf{100} & \textbf{100} & \textbf{100} & \textbf{100} & 80.7 & 52.9 & \textbf{100} & 68.1 & 96.1 & 50.3 \\

LLM-PV: \texttt{GPT-5-Thinking} (no tools)
& \textbf{100} & 51.1 & 78.4 & 64.7 & 50.1 & 50.4 & 50.0 & 79.9 & 50.2 & 52.9 & 64.0 & 51.0 & 50.2 \\

LLM-PV: \texttt{GPT-5-Thinking} (with tools)
& \textbf{100} & \textbf{100} & \textbf{100} & \textbf{100} & \textbf{100} & \textbf{100} & \textbf{100} & \textbf{100} & \textbf{100} & \textbf{100} & \textbf{80.6} & \textbf{96.9} & 50.1 \\

\bottomrule
\end{tabular}%
}
\caption{
\textbf{Test accuracy at length $n{=}100$ across baselines and LLM-PV.}
We compare classic ML algorithms, training from scratch (SGD on 1B-scale LMs), fine-tuning (top-$k$ layers or full) of 1B-scale pretrained LMs, and in-context learning with 30B+ instruction-tuned models. {\bf Most baselines remain near chance on the algorithmic tasks at $n{=}100$, while LLM-PV achieves high test accuracy on all tasks except SHA.}
}
\label{tab:all_methods_n100}
\end{table*}

\begin{table*}[t]
\centering
\small
\setlength{\tabcolsep}{3.9pt}
\renewcommand{\arraystretch}{1.06}

\begin{adjustbox}{max width=\textwidth,center}
\begin{tabular}{@{}l|ccccc|ccccc|ccccc|ccccc@{}}
\toprule
& \multicolumn{5}{c|}{\textbf{SGD}}
& \multicolumn{5}{c|}{\textbf{FT}}
& \multicolumn{5}{c|}{\textbf{ICL}}
& \multicolumn{5}{c}{\textbf{LLM-PV}} \\
\cmidrule(lr){2-6}\cmidrule(lr){7-11}\cmidrule(lr){12-16}\cmidrule(lr){17-21}
\textbf{Task}
& $n{=}20$ & $n{=}25$ & $n{=}30$ & $n{=}50$ & $n{=}100$
& $n{=}20$ & $n{=}25$ & $n{=}30$ & $n{=}50$ & $n{=}100$
& $n{=}20$ & $n{=}25$ & $n{=}30$ & $n{=}50$ & $n{=}100$
& $n{=}20$ & $n{=}25$ & $n{=}30$ & $n{=}50$ & $n{=}100$ \\
\midrule

FP
& 56.3 & 56.2 & 52.4 & 52.7 & 52.2
& 50.8 & 51.0 & 50.4 & 51.2 & 51.2
& 52.0 & 47.0 & 54.0 & 51.0 & 53.0
& \alln & \alln & \alln & \alln & \alln \\

FHP
& 51.0 & 52.5 & 50.7 & 53.7 & 50.6
& 52.7 & 51.7 & 50.8 & 50.5 & 51.3
& 50.0 & 50.0 & 46.0 & 59.0 & 54.0
& \textbf{100} & \textbf{100} & \textbf{100} & \textbf{100} & \textbf{100} \\

3-P
& 50.6 & 50.5 & 52.2 & 51.4 & 54.3
& 50.9 & 50.7 & 53.3 & 51.0 & 51.0
& 54.0 & 56.0 & 51.0 & 48.0 & 50.0
& \textbf{100} & \textbf{100} & \textbf{100} & \textbf{100} & \textbf{100} \\

10-P
& 50.6 & 50.2 & 50.3 & 50.6 & 50.3
& 51.1 & 50.5 & 50.8 & 50.3 & 50.7
& 51.0 & 55.0 & 55.0 & 51.0 & 50.0
& \textbf{100} & \textbf{100} & \textbf{100} & \textbf{100} & \textbf{100} \\

Pattern1
& 91.4 & 82.8 & 57.8 & 58.7 & 51.5
& 74.6 & 72.0 & 68.4 & 68.5 & 61.5
& 89.0 & 69.0 & 77.0 & 61.0 & 56.0
& \alln & \textbf{98.9} & \textbf{98.5} & \alln & \alln \\

Pattern2
& 97.5 & {\bf 96.0} & {\bf 95.2} & 74.5 & 55.1
& 69.4 & 69.2 & 67.4 & 67.3 & 59.4
& 79.0 & 68.0 & 72.0 & 79.0 & 51.0
& \alln & 94.2 & 93.2 & \alln & \alln \\

IsPal
& 62.7 & 56.5 & 54.8 & 51.6 & 51.6
& 52.4 & 51.3 & 51.1 & 50.2 & 49.7
& 58.0 & 53.0 & 63.0 & 52.0 & 51.0
& \textbf{100} & \textbf{96.0} & \textbf{100} & \textbf{100} & \textbf{100} \\

Dyck-2$^{*}$
& 65.7 & 59.1 & 55.9 & 51.1 & 51.8
& 64.6 & 57.9 & 53.9 & 52.8 & 52.7
& 70.0 & 57.0 & 61.0 & 54.0 & 54.0
& \textbf{77.4} & \textbf{90.5} & \textbf{80.0} & \textbf{90.5} & \textbf{100} \\


IsPrime
& 59.9 & 60.2 & 63.0 & 57.0 & 58.8
& 59.2 & 59.8 & 59.2 & 62.5 & 58.9
& 53.0 & 53.0 & 53.0 & 57.0 & 53.0
& \alln & \alln & \alln & \alln & \alln \\




CAP$^{\diamond}$
& 52.5 & 51.4 & 52.0 & 50.5 & 50.0
& 50.9 & 50.6 & 51.2 & 51.1 & 50.8
& 52.0 & 47.0 & 46.0 & 50.0 & 49.0
& \textbf{100} & \textbf{100} & \alln & \alln & \alln \\


\bottomrule
\multicolumn{21}{@{}l}{\footnotesize
100$^{\dagger}$ indicates 100\% for all $n$.\quad
$^{*}$For Dyck-2, lengths are $n\in\{20,40,60,80,100\}$ respectively.
}
\end{tabular}
\end{adjustbox}

\caption{
\textbf{LLM-PV vs.\ training and prompting baselines.} Test accuracy (\%) for input lengths $n\in\{20,25,30,50,100\}$. \textbf{SGD}: best test accuracy across \texttt{Qwen3-1.7B}, \texttt{Deepseek-Coder-1.3B}, and \texttt{Llama3.2-1B}, each trained from scratch. \textbf{FT}: best test accuracy across \texttt{Llama3.2-1B}, \texttt{Qwen3-1.7B}, and \texttt{Deepseek-Coder-1.3B}, sweeping tuned layers (\{top-2, top-4, top-8, full\}) and training hyperparameters, separately for each (task,$n$). \textbf{ICL}: best test accuracy across \texttt{Qwen3-30B-Instruct}, \texttt{Qwen3-Coder-30B-Instruct}, and \texttt{Deepseek-Coder-33B-Instruct}.
}
\label{tab:pv_vs_all}
\end{table*}


\begin{table}[t]
\centering
\small
\setlength{\tabcolsep}{1.2pt}
\renewcommand{\arraystretch}{0.9}
\begin{adjustbox}{max width=0.48\textwidth,center}
\begin{tabular}{@{}l*{12}{c}@{}}
\toprule
\textbf{Task}
& \multicolumn{2}{c}{$n{=}20$}
& \multicolumn{2}{c}{$n{=}25$}
& \multicolumn{2}{c}{$n{=}30$}
& \multicolumn{2}{c}{$n{=}50$}
& \multicolumn{2}{c}{$n{=}100$}
& \multicolumn{2}{c}{$\mathbf{PAC\ bound}$} \\
\cmidrule(lr){2-3}\cmidrule(lr){4-5}\cmidrule(lr){6-7}\cmidrule(lr){8-9}\cmidrule(lr){10-11}\cmidrule(lr){12-13}
& \textbf{Train} & \textbf{Test}
& \textbf{Train} & \textbf{Test}
& \textbf{Train} & \textbf{Test}
& \textbf{Train} & \textbf{Test}
& \textbf{Train} & \textbf{Test}
& \textbf{ASCII} & \textbf{PyTok} \\
\midrule
\textbf{10-P}     & 100 & 53.9 & 100 & 49.8 & 100 & 50.5 & 100 & 49.2 & 100 & 50.7 & 0.995 & 0.998 \\
\textbf{CAP}      & 100 & 99.9 & 100 & 50.3 & 100 & 50.2 & 100 & 49.5 & 100 & 50.4 & 0.985 & 0.994 \\
\textbf{IsPrime} & 100 & 59.8 & 100 & 58.7 & 100 & 60.3 & 100 & 60.1 & 100 & 59.9 & 0.992 & 0.997 \\
\bottomrule
\end{tabular}
\end{adjustbox}
\caption{\textbf{BLOOM-75M (SGD) train/test accuracy (\%) and a program-length PAC bound for training with $m{=}100$k samples.}
We report train/test accuracy at each length $n$.
The last two columns report the bound
$1-m^{-1}[L\log|\Sigma|+\log(2L^2/\delta)]$
with $\delta{=}10^{-10}$, using $L$ from a compact reference implementation measured in ASCII bytes
($|\Sigma|{=}128$) or Python lexical tokens ($|\Sigma|{=}64$); the bound is independent of $n$.
{\bf Despite perfect training accuracy at each $n$, test accuracy stays near chance on 10-P and CAP, even though both admit short programs for which the PAC-learning bounds are close to 1.}}
\label{tab:bloom_ablation}
\end{table}

\begin{figure}[t]
  \centering
  \begin{tabular}{ccc}
    \includegraphics[width=0.29\linewidth]{\detokenize{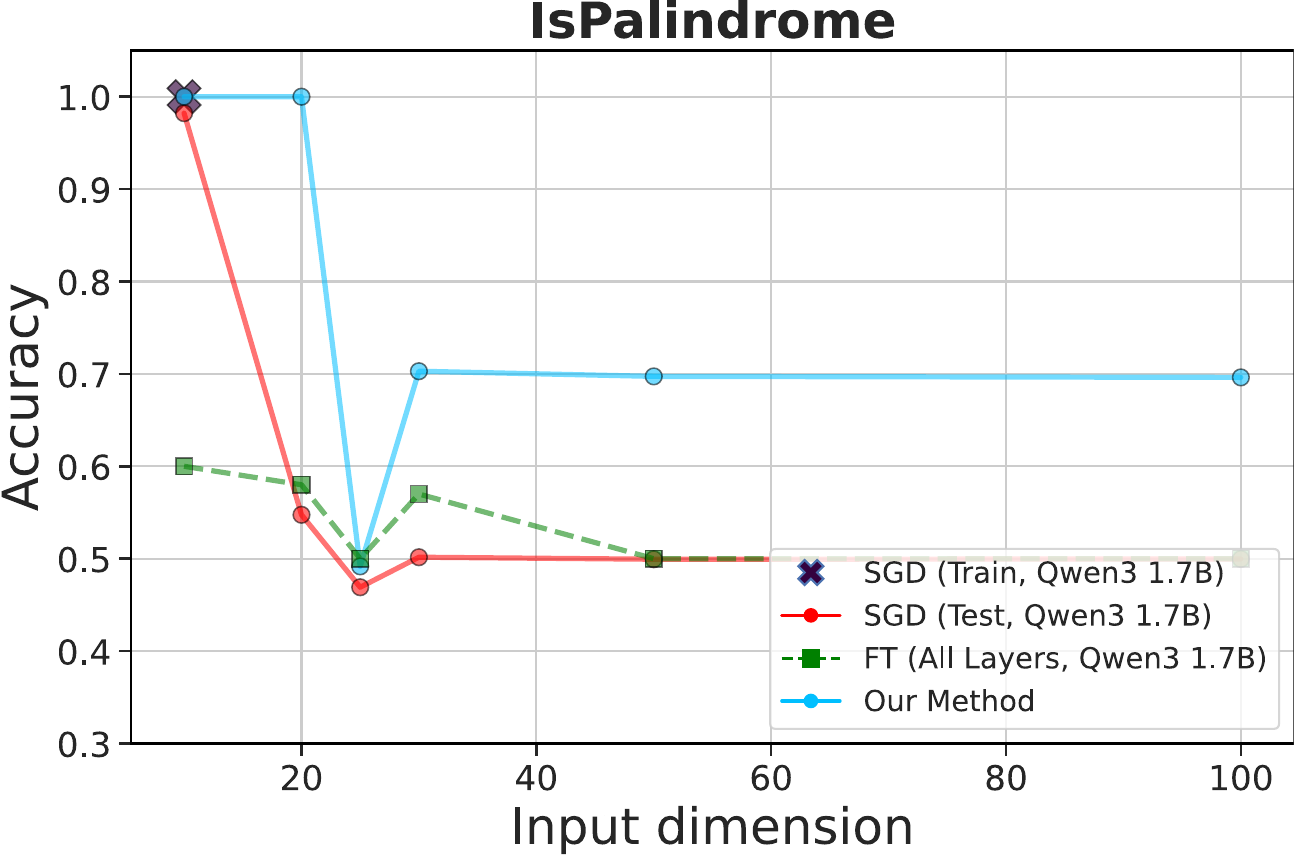}} &
    \includegraphics[width=0.29\linewidth]{\detokenize{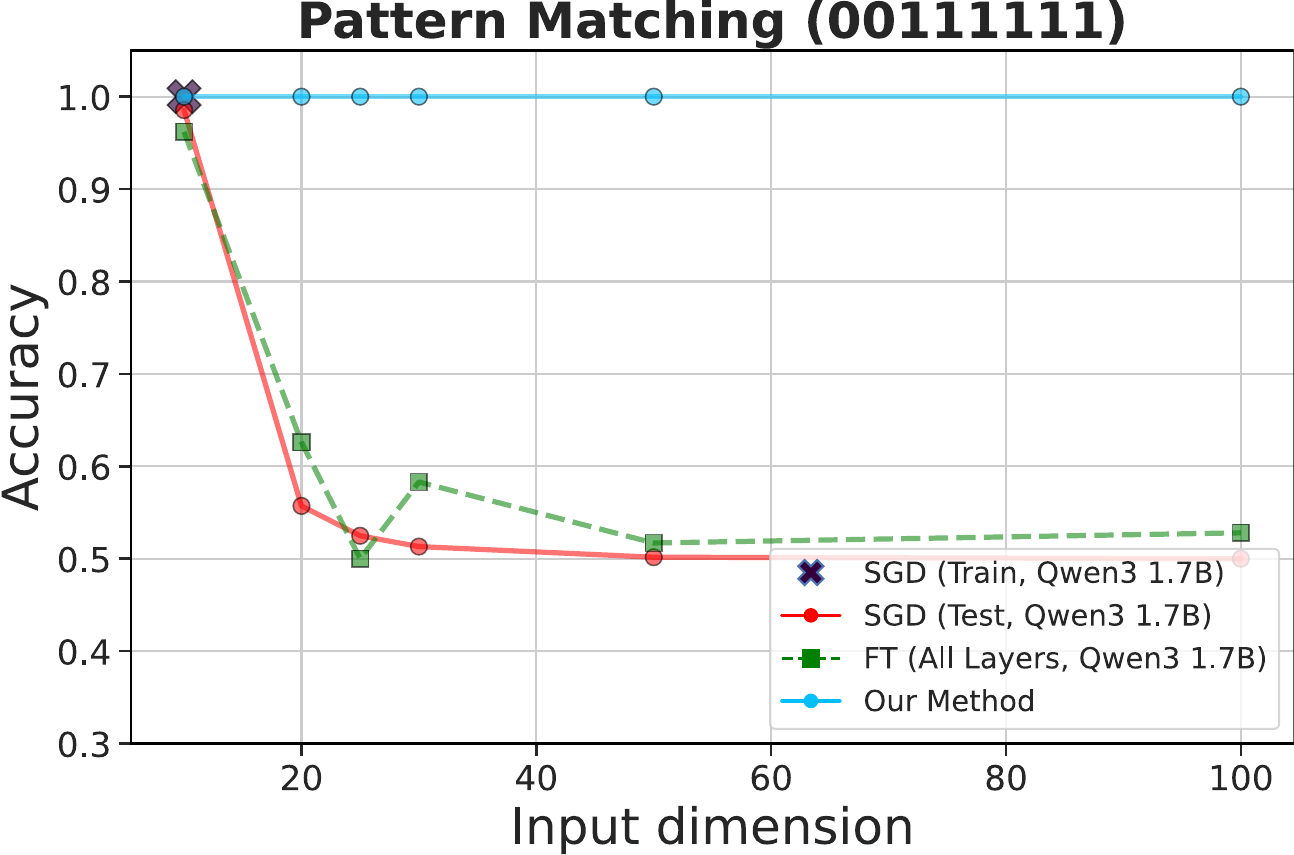}} &
    \includegraphics[width=0.29\linewidth]{\detokenize{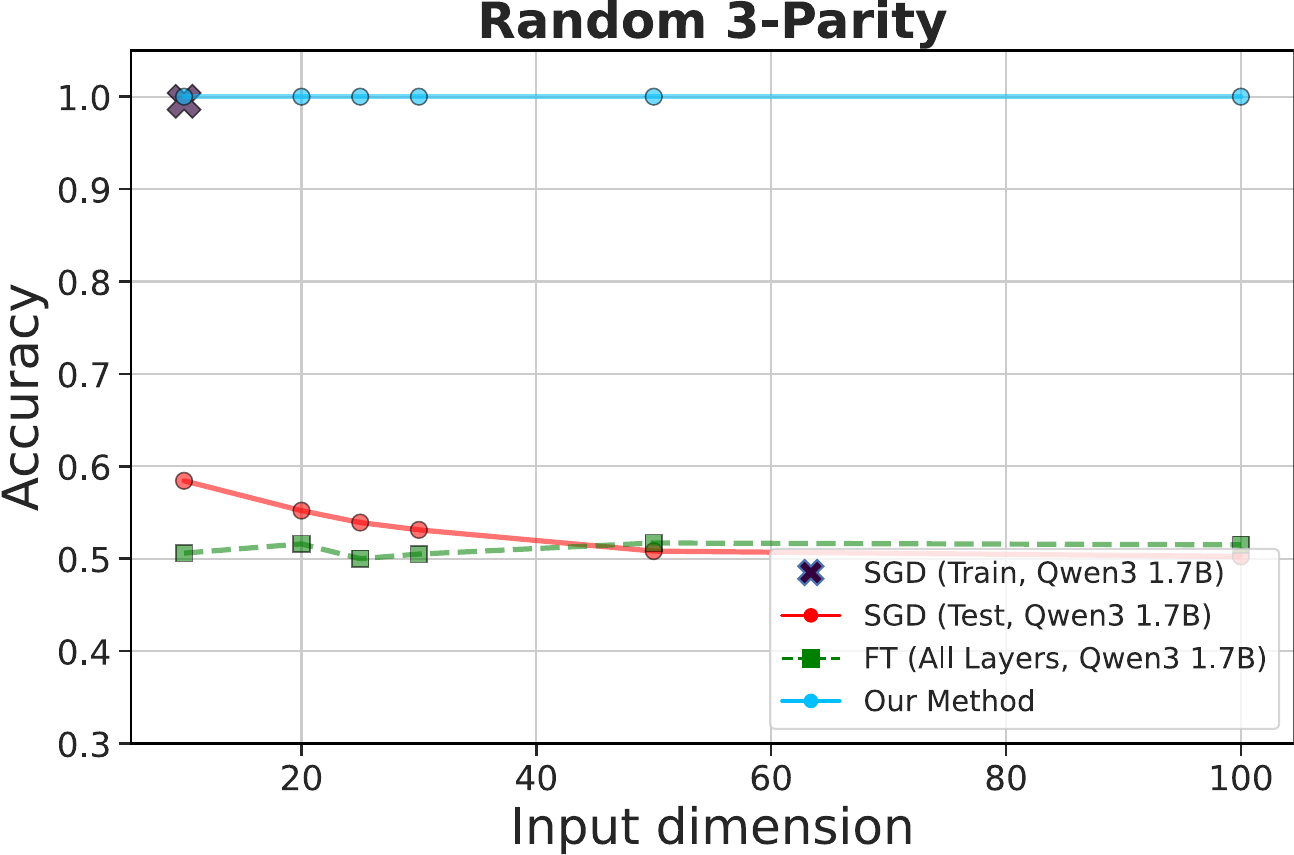}} 
  \end{tabular}
\caption{\textbf{LLM-PV generalizes to much larger input lengths.} All methods (SGD, FT, and \textsc{LLM-PV}) are trained on 200 examples at a fixed length ($n{=}10$) and evaluated on lengths up to $n{=}100$. SGD learns the training distribution but degrades to near-chance accuracy (${\approx}50\%$) as $n$ grows. In contrast, \textsc{LLM-PV} maintains high accuracy and, on Pattern Matching and Random 3-Parity, synthesizes dimension-invariant programs.}
\label{fig:length_generalize_ablation} 
\end{figure}

\begin{figure}[t]
  \centering
  \begin{tabular}{ccc}
    \includegraphics[width=0.29\linewidth]{\detokenize{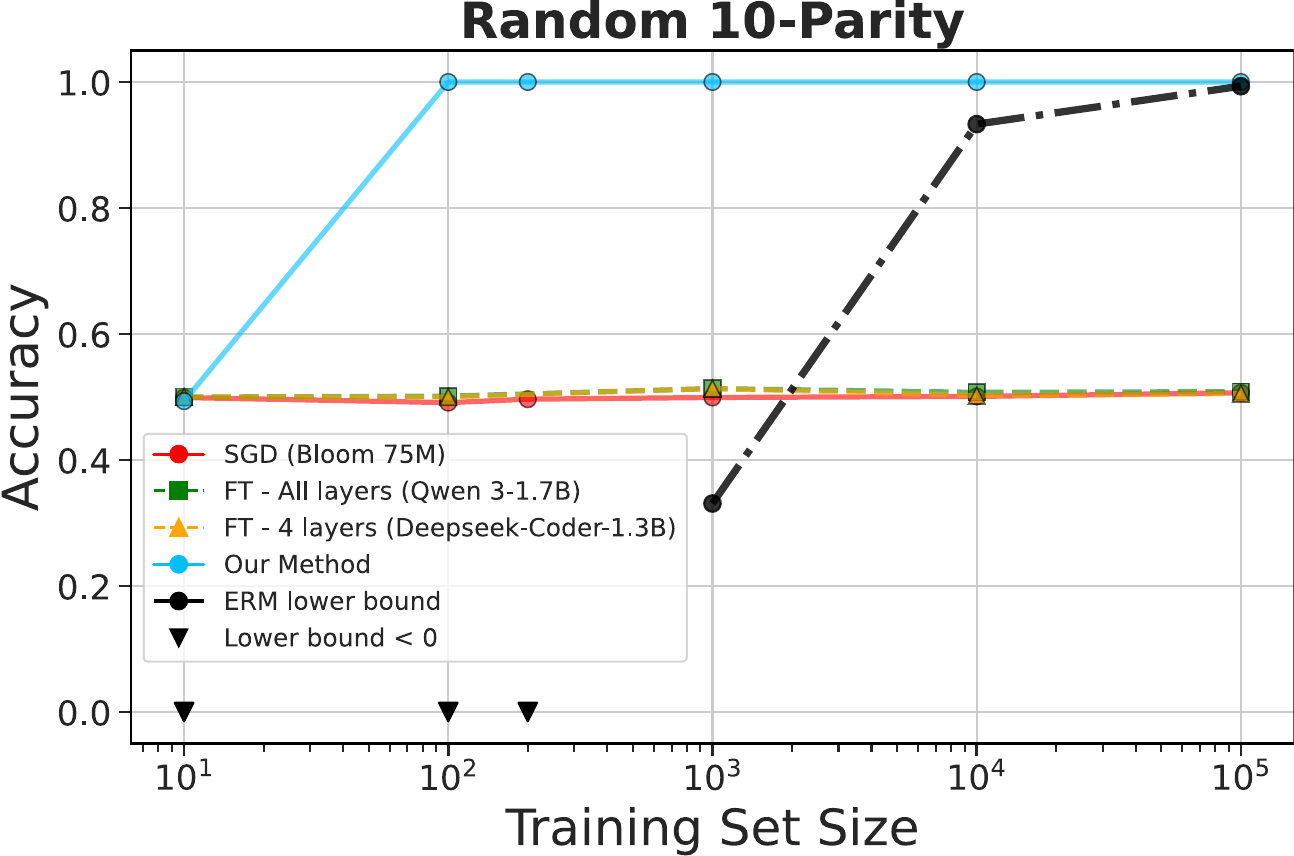}} &
    \includegraphics[width=0.29\linewidth]{\detokenize{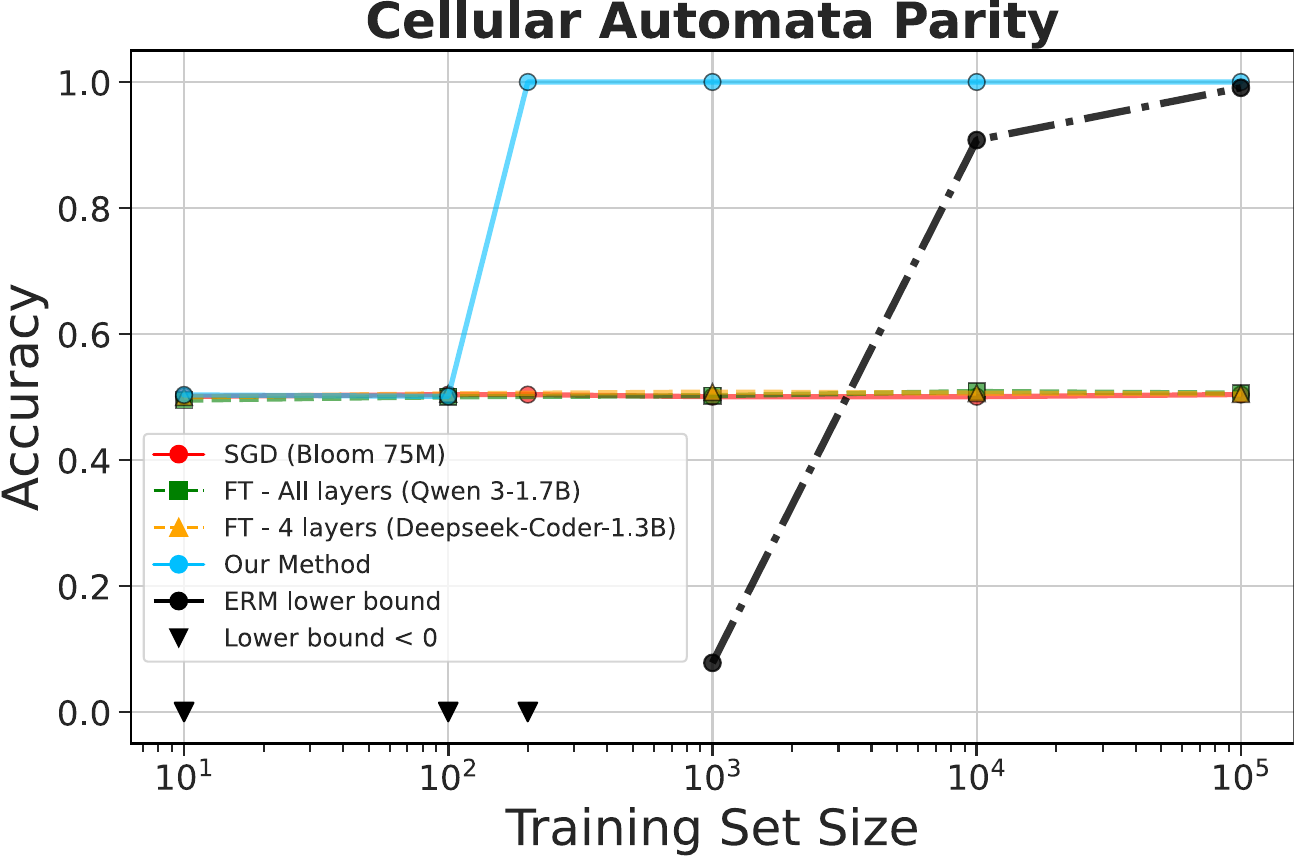}} &
    \includegraphics[width=0.29\linewidth]{\detokenize{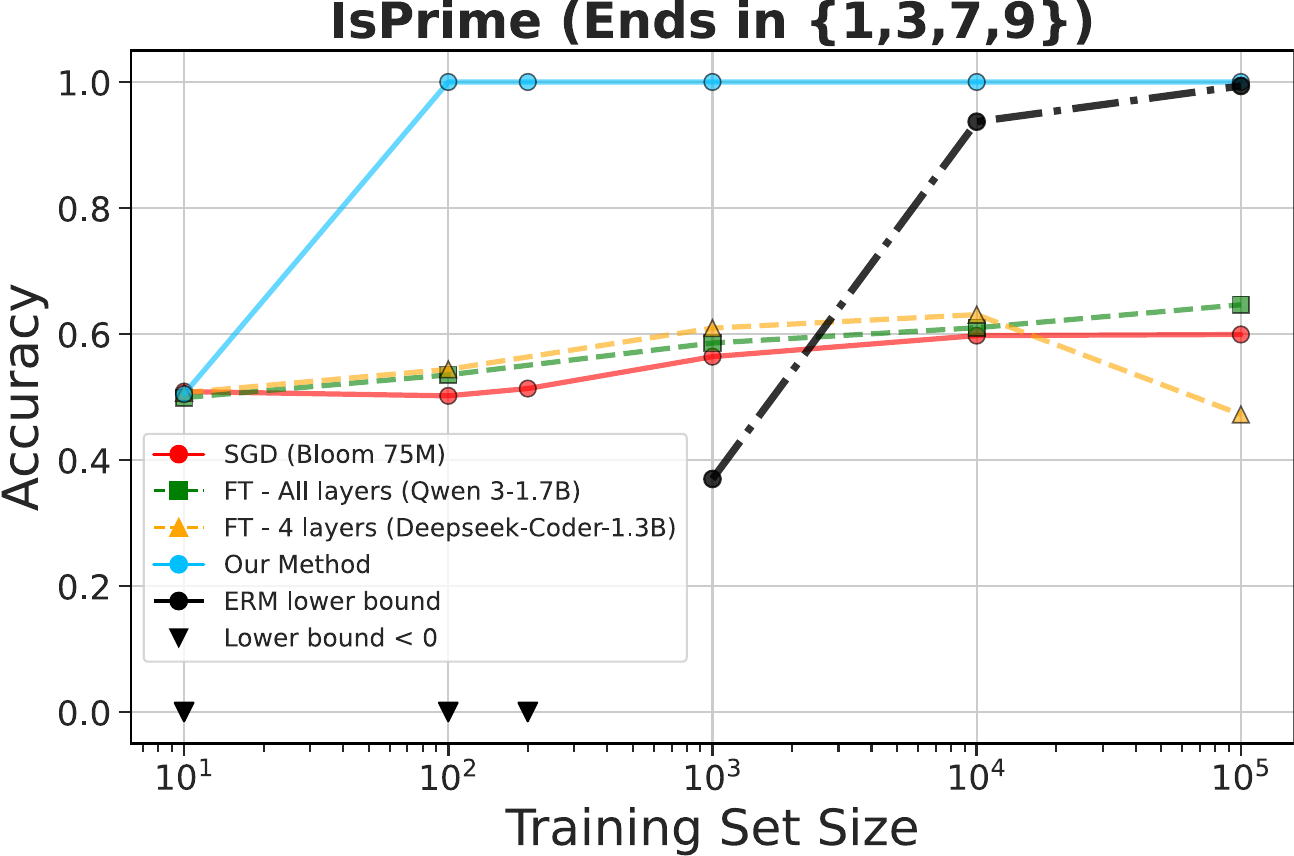}} \\
  \end{tabular}
\caption{{\bf SGD-trained LLMs fail to learn algorithmic tasks even with 100k examples; LLM-PV succeeds with 200.} We compare the test performance of LLM-PV and SGD training (from-scratch and fine-tuning) on 10-P {\bf (left)}, CAP {\bf (middle)}, and IsPrime {\bf (right)} when varying the training set size. {\bf SGD training overfits, with test accuracy near chance across all training sizes. LLM-PV achieves 100\% test accuracy with just 200 examples.}}
\label{fig:size_ablation}
\end{figure}

\begin{table}[t]
\centering
\small
\setlength{\tabcolsep}{4pt}
\renewcommand{\arraystretch}{1.15}
\begin{adjustbox}{max width=0.48\textwidth,center}
\begin{tabular}{@{}lccccccc@{}}
\toprule
\textbf{Model} &
\textbf{\shortstack{Adult\\Income}} &
\textbf{\shortstack{Secondary\\Mushroom}} &
\textbf{\shortstack{CDC Diabetes\\Health Indicators}} & 
\textbf{\shortstack{HRTU2}} &
\textbf{\shortstack{Chess\\(KR vs KP)}} &
\\
\midrule
\texttt{SVM}           & $52.8 \pm 0.000$ & $68.1 \pm 0.000$ & $69.1 \pm 0.000$ & 
$91.5 \pm 0.000$ & $89.4 \pm 0.000$ & 
\\
\texttt{GA}           & $75.2 \pm 0.016$ & $64.7 \pm 0.020$ & $69.0 \pm 0.014$ & 
$90.6 \pm 0.001$ & $92.5 \pm 0.013$ & 
\\
\texttt{Decision Tree} & $67.8 \pm 0.020$ & $69.0 \pm 0.016$ & $62.5 \pm 0.004$ & 
$89.1 \pm 0.002$ & $\mathbf{94.1 \pm 0.003}$ & 
\\
\texttt{Random Forest} & $74.2 \pm 0.008$ & $70.8 \pm 0.003$ & $\mathbf{72.7 \pm 0.006}$ & 
$91.4 \pm 0.005$ & $88.1 \pm 0.011$ & 
\\
\texttt{XGBoost}       & $71.4 \pm 0.006$ & $66.1 \pm 0.012$ & $68.6 \pm 0.004$ & 
$91.4 \pm 0.001$ & $93.1 \pm 0.005$ & 
\\
\texttt{TabPFN}           & $\mathbf{77.7 \pm 0.000}$ & $69.9 \pm 0.000$ & $72.1 \pm 0.000$ & 
$\mathbf{93.1 \pm 0.000}$ & $93.1 \pm 0.000$ & 
\\
\midrule
LLM-PV        & $75.8 \pm 0.007$ & $\mathbf{72.2 \pm 0.020}$ & $71.1 \pm 0.002$ & 
$92.3 \pm 0.007$ & $92.3 \pm 0.012$ & 
\\
\bottomrule
\end{tabular}
\end{adjustbox}

\caption{\textbf{Tabular benchmarks.} Test accuracy (\%) for comparing classic ML baselines with LLM-PV. Our LLM-PV method achieves comparable results to other standard tabular learners.}
\label{tab:tabular}
\end{table}


\subsection{Setup}\label{sec:setup}

We evaluate whether LLM-guided propose-and-verify can recover short executable rules from few labeled examples, and compare it to standard learning baselines.
Each task is binary classification with inputs $x\in\mathcal X$ (bitstrings/digit strings/graph encodings) and labels $y(x)\in\{0,1\}$. We draw $m{=}200$ i.i.d.\ labeled examples and split them into $S_{\mathrm{tr}}$ and $S_{\mathrm{val}}$ of size 100 each. We report test accuracy on an independent held-out set of 10{,}000 examples.

\subsubsection{Evaluation Tasks} 

{\bf Synthetic tasks.\enspace} We use synthetic algorithmic tasks spanning non-local properties and simple heuristics: {\bf (1)} parity variants: full parity (FP), first-half parity (FHP), random 3-parity (3-P), random 10-parity (10-P); {\bf (2)} pattern matching for fixed motifs \texttt{00111111} (Pattern1) and \texttt{10101010} (Pattern2);
{\bf (3)} palindrome detection (IsPal); {\bf (4)} Dyck-2 membership (Dyck-2); {\bf (5)} primality (IsPrime) while restricting negatives to last digits $\in \{1,3,7,9\}$; {\bf (6)} the shifted-input primality variant (IsPrime+47), where we add 47 to each integer before labeling; {\bf (7)} cellular automata parity (CAP); and {\bf (8)} SHA-256 parity (SHA). All datasets are class-balanced by construction. Formal definitions and data generation appear in App.~\ref{app:tasks}.

Despite being familiar benchmarks, these targets are difficult to learn from i.i.d.\ labeled examples without an explicit algorithmic bias. First, $k$-parity suffers from a combinatorial search barrier: the label depends on a tiny unknown subset of coordinates, and there are $\binom{100}{10}$ candidate 10-bit subsets over 100 variables, so fitting low-order correlations provides essentially no guidance toward the correct rule. Second, palindromes and Dyck\textnormal{-}2 have hard negatives that differ from positives by only a few edits, which preserves most local $n$-gram statistics; thus, any learner that relies primarily on local cues will struggle to separate the classes. For Dyck\textnormal{-}2, we additionally encode the two bracket types as a 2-bit alphabet $\{00,01,10,11\}$ to remove trivial formatting cues; success must come from tracking the underlying stack-like constraint rather than from recognizable characters. Third, motif detection is a needle-in-a-haystack problem: a short pattern can occur at an unknown position within a long string, so generalization requires reasoning that is robust to shifts rather than exploiting fixed coordinates. 
Fourth, cellular-automata parity and SHA-based labels exhibit an avalanche effect, where small input changes can induce global output changes, ruling out simple heuristic rules based on a few digits or local patterns. Finally, for 100-digit primality, memorization is implausible given the astronomically large number of candidates (on the order of $10^{97}$ primes), so systematic success requires implementing an algorithmic test; moreover, IsPrime+47 makes this necessity explicit by defining the label via the primality of $\mathrm{int}(x)+47$, so any strategy that treats the label as a direct property of the observed digit string $x$ will fail.



\subsection{Methods Compared}

We compare \textsc{LLM-PV} to several families of representative baselines. To keep the main text focused, we describe each baseline at a high level here and defer full hyperparameters, prompts, and sweeps to App.~\ref{app:baselines}.

{\bf LLM-PV.\enspace} We instantiate Alg.~\ref{alg:llm-ktry} with several backbone models, including
open-source models such as \texttt{Qwen3-30B-Instruct},
\texttt{Deepseek-Coder-33B-Instruct}, and \texttt{Olmo-3-32B-Instruct}, as well
as closed models such as \texttt{Gemini-3.1-Pro} and \texttt{GPT-5-Thinking}.
The open-source models are evaluated without tool calls. For
\texttt{GPT-5-Thinking}, we explicitly evaluate both no-tool and tool-augmented
variants, allowing us to isolate the effect of tool use. Given
$S_{\mathrm{tr}}$, the model samples up to $k{=}5$ candidate Python programs
under a fixed prompt (Fig.~\ref{fig:prompt}). We discard invalid or duplicate
programs, execute the remaining candidates in a sandboxed interpreter, and
select the program with the lowest validation error on $S_{\mathrm{val}}$
(Fig.~\ref{fig:experiment-flow-validation}). We do not use validation feedback
to adapt the sampling distribution.

{\bf Baselines.\enspace} (1) \emph{Classic ML:} XGBoost, Random Forest, SVM, and a Genetic Algorithm baseline, each with validation-based model and hyperparameter selection. (2) \emph{A pre-trained tabular foundation model:} TabPFN~\cite{hollmann2022tabpfn}. (3) \emph{From scratch:} \texttt{Qwen3-1.7B} trained from scratch as a classifier on the $m{=}200$ labeled examples. (4) \emph{Fine-tuning:} \texttt{Qwen3-1.7B}, \texttt{Llama3.2-1B}, and \texttt{Deepseek-Coder-1.3B} fine-tuned on the same $m{=}200$ labeled examples; hyperparameters (e.g., learning rate, batch size, and the number of tuned layers) are selected via a grid sweep. (5) \emph{In-context learning:} We evaluate large language models prompted with
all $m{=}200$ labeled examples and queried for the test labels. The comparison
includes open-weight models, code-specialized models, and closed frontier models.
Open-weight models are run without tool use. For \texttt{GPT-5-Thinking}, we
report both no-tool and tool-augmented variants; for \texttt{Gemini-3.1-Pro},
we report the tool-augmented variant. (6) \emph{Agentic ML baselines:} We also compare against general-purpose
agentic ML systems, including \texttt{MLAgentBench}~\cite{10.5555/3692070.3692884} and \texttt{AIDE-ML}~\cite{jiang2025aideaidrivenexplorationspace}. These
baselines are instantiated with \texttt{GPT-5-Thinking} and tool use, allowing
the agent to write, execute, and revise code during the search process. All
prompting, tool-use, and evaluation details are provided in
App.~\ref{app:baselines}.

\subsection{Main results}
\label{sec:main_results}

{\bf LLM baselines often fit without recovering the rule.\enspace}
Across the core algorithmic tasks (parity variants, Dyck-2, CAP), both training from scratch and fine-tuning remain close to chance at $n{=}100$ (Tab.~\ref{tab:all_methods_n100}), suggesting length-specific memorization rather than rule recovery. Concretely, for the parity variants $\{\mathrm{FP},\mathrm{FHP},\mathrm{3\text{-}P},\mathrm{10\text{-}P}\}$, SGD/FT scores fall in the range $48.6$--$54.3\%$ at $n{=}100$ (Tab.~\ref{tab:all_methods_n100}); for CAP they are even tighter at $49.4$--$50.8\%$ (Tab.~\ref{tab:all_methods_n100}). Fine-tuning can exploit local structure when it exists: on Pattern1 at $n{=}100$, the best FT model reaches $61.5\%$ versus $51.5\%$ for the best SGD model (Tab.~\ref{tab:all_methods_n100}). However, this improvement does not transfer to the parity variants or CAP, where FT remains near chance (Tab.~\ref{tab:all_methods_n100}). In-context learning at 30B scale yields occasional gains but still does not produce consistent algorithmic generalization: at $n{=}100$, ICL is the strongest baseline on SHA (up to $54.0\%$; Tab.~\ref{tab:all_methods_n100}), while remaining near chance on parity and CAP (Tab.~\ref{tab:all_methods_n100}).

{\bf Tabular learners capture surface-level cues.\enspace} At $n{=}100$, tabular learners can succeed when local statistics are predictive (e.g., IsPal reaches $69.7$--$87.1\%$ for SVM/GA; IsPrime reaches $51.7$--$61.8\%$ for GA/XGBoost; Tab.~\ref{tab:all_methods_n100}). Notably, even TabPFN~\cite{hollmann2022tabpfn}---a strong pre-trained tabular foundation-model---does not close the gap on our algorithmic tasks: at $n{=}100$ it remains at (near) chance on FP/FHP/3-P/10-P, Dyck-2, and CAP (roughly $48$--$51\%$ across these columns; Tab.~\ref{tab:all_methods_n100}). Overall, while tabular methods can exploit shallow cues when they exist, they fail on parity, Dyck-2, and CAP, highlighting a qualitative distinction between feature-based prediction and learning an algorithmic procedure.

{\bf \textsc{LLM-PV} recovers algorithmic rules.\enspace}
In contrast, \textsc{LLM-PV} achieves perfect accuracy at $n{=}100$ on all parity variants, IsPalindrome, Dyck-2, CAP, and IsPrime (all 100\%; Tab.~\ref{tab:all_methods_n100}), indicating broad rule recovery rather than task-specific pattern fitting. More strikingly, \textsc{LLM-PV} maintains high accuracy on the shifted-primality task IsPrime+47, reaching $80.6\%$ at $n{=}100$ (Tab.~\ref{tab:all_methods_n100}). Since adding a fixed offset breaks many superficial correlations while preserving a clean arithmetic relationship, this result suggests \textsc{LLM-PV} is not merely memorizing primality patterns but is often discovering a transferable computational structure. SHA remains difficult in our setup: all methods are close to chance, with the best performance coming from ICL (up to $54.0\%$ at $n{=}100$; Tab.~\ref{tab:all_methods_n100}).

{\bf Direct prompting is the wrong output space.\enspace}
ICL asks the model to behave like the target function directly: given a new
input, output a label. This is different from recovering an executable rule.
Consistent with this distinction, even tool-augmented ICL remains unreliable:
\texttt{GPT-5-Thinking} with tools improves over its no-tool variant on several
tasks, such as FP ($99.0\%$ vs.\ $82.0\%$), FHP ($88.0\%$ vs.\ $66.0\%$), and
10-P ($79.0\%$ vs.\ $47.0\%$), but it still remains far from rule recovery on
Pattern1, Pattern2, Dyck-2, CAP, CycleCount, and SHA
(Tab.~\ref{tab:all_methods_n100}). Thus, tool-augmented direct prediction does not reliably produce algorithmic generalization.

{\bf Tools help when they support program search.\enspace}
The \texttt{GPT-5-Thinking} ablation inside \textsc{LLM-PV} shows that tool use
is highly beneficial when it is embedded in a propose-and-verify loop. Without
tools, \textsc{LLM-PV} recovers some structure but remains incomplete, e.g.,
$78.4\%$ on Pattern1, $64.7\%$ on Pattern2, and $79.9\%$ on Dyck-2. With tools,
the same backbone reaches $100\%$ accuracy on 10 of the 13 tasks and remains
high on IsPrime+47 ($80.6\%$) and CycleCount ($96.9\%$)
(Tab.~\ref{tab:all_methods_n100}). This suggests that much of the gain comes from programmatic search: the agent can write code, execute it, inspect failures, test hypotheses, and refine candidate solutions.

{\bf Tool use alone is not sufficient.\enspace}
The comparison with ICL shows that tools are useful only when paired with the
right search objective. Tool-augmented ICL has access to the same kind of
external computation, but it still returns labels rather than an executable
rule. In contrast, \textsc{LLM-PV} uses tools to search over programs and then
selects among them using validation error. The gap between tool-augmented ICL
and tool-augmented \textsc{LLM-PV} shows that the benefit is not merely tool
access; it is tool use organized around program discovery.

{\bf Generic agents are not enough; the output matters.\enspace}
The agentic ML baselines \texttt{MLAgentBench} and \texttt{AIDE-ML}, both
instantiated with \texttt{GPT-5-Thinking} and tool use, remain close to chance
on most tasks. For example, \texttt{AIDE-ML} stays between roughly $49\%$ and
$55.5\%$ across all tasks, while \texttt{MLAgentBench} also fails broadly and
even drops below chance on several tasks (Tab.~\ref{tab:all_methods_n100}).
These systems search over standard ML workflows and return trained predictors,
whereas our tasks often require an explicit algorithmic rule. \textsc{LLM-PV}
succeeds because its hypothesis class is executable programs, not generic
predictors.

{\bf Open-weight models show partial program discovery without tools.\enspace}
The no-tool open-weight \textsc{LLM-PV} variants occasionally recover individual
rules, but their successes are sparse and unstable across tasks. For example,
\texttt{DeepSeek-Coder-33B-Instruct} and \texttt{OLMo-3-32B-Instruct} solve FP,
\texttt{OLMo-3-32B-Instruct} reaches $83.5\%$ on Pattern2, and
\texttt{Qwen3-30B-Instruct} reaches $80.4\%$ on Dyck-2
(Tab.~\ref{tab:all_methods_n100}). However, most other entries remain near
chance, suggesting that the propose-and-verify formulation can expose partial
rule-recovery ability without tools, but reliable recovery requires active
testing and debugging.

{\bf More data does not reliably fix SGD on these tasks.\enspace}
To test whether these failures are a small-data artifact, we train \texttt{BLOOM-75M} from scratch with SGD on $m{=}100$k examples per task (separately for each $n$).
The model again fits the training data perfectly, yet test accuracy remains near chance on 10-P and on CAP for $n\ge 25$, and is only modest on IsPrime (Tab.~\ref{tab:bloom_ablation}).
Fig.~\ref{fig:size_ablation} further shows that scaling $m$ from 10 to 100k does not yield reliable generalization for the SGD baseline in our setup, whereas \textsc{LLM-PV} attains perfect test accuracy with $m{=}200$ on these tasks.

{\bf Length generalization.\enspace} In Fig.~\ref{fig:length_generalize_ablation} we study out-of-distribution length generalization by training on 200 samples at a short input length ($n{=}10$) and evaluating on longer sequences. We keep the SGD (and FT) setup identical to our primary experiments. The SGD and FT baselines are oftentime able to generalize at $n{=}10$ but fail to generalize across lengths: accuracy rapidly collapses toward ${\approx}50\%$ as $n$ increases. In contrast, \textsc{LLM-PV} often synthesizes dimension-invariant programs, yielding stable performance across lengths, including perfect generalization on tasks such as Pattern Matching and Random 3-Parity.

{\bf LLM-PV performs comparably to strong tabular ML baselines.\enspace}
In addition to the tasks above, we also consider tabular datasets. For each task, we use only 100 fixed training samples and 100 fixed validation samples, and we report test error on the original test set. To ensure that the samples are not easily recognized by LLM-based methods (e.g., LLM-PV), we preprocess each input by applying a fixed random linear transformation $x \mapsto Wx+b$, where $W$ is diagonal and $b$ is a vector, with entries drawn i.i.d.\ from a standard normal distribution. As per Tab.~\ref{tab:tabular}, \textsc{LLM-PV} is competitive against standard tabular learners, suggesting that propose-and-verify can remain effective beyond synthetic sequence tasks. 

\section{Discussion}

We studied \textsc{LLM-PV}, a learning algorithm that uses a pretrained LLM as a proposal prior and selects among executable programs by validation error. Across algorithmic tasks, \textsc{LLM-PV} often recovers exact rules from few examples and generalizes far beyond the training length. This suggests a different use of pretrained LLMs in learning: not as predictors that directly imitate the target function, but as search priors over discrete, verifiable hypotheses.

The main lesson is that the output space matters. Standard gradient-based predictors, fine-tuned models, and in-context learning ultimately produce labels or learned predictors. By contrast, \textsc{LLM-PV} searches over executable rules and uses held-out error only to select among candidate programs. This separation keeps the learning criterion simple and auditable: the LLM proposes programs, while execution and validation performance decide.

More broadly, our results point toward \emph{learning via LLM-guided search with verification}. Future work should characterize when the LLM proposal distribution places enough mass on correct or near-correct programs, when validation can reliably identify the right hypothesis, and where the approach breaks down.

The main limitation is scope. Our experiments focus on structured low-data regimes, and \textsc{LLM-PV} is not intended as a general replacement for supervised deep learning. Its success depends on a capable agentic proposer and on sufficient proposal mass near the correct program. As target programs become longer, more compositional, or less aligned with the pretrained model's prior, this mass may decrease, requiring more samples, more LLM calls, or richer propose-verify-refine procedures. Extending the approach to larger synthesis problems and real-world structured domains remains future work.


\newpage

\section*{Impact Statement}

This work studies program learning in controlled algorithmic settings. It raises no direct ethical, safety, or environmental concerns in the scope considered here. We introduce LLM-PV, a propose-and-verify framework that performs ERM-style selection over a discrete program class using a pretrained model only as a proposal prior, and we provide empirical evidence that it can recover compact rules from few labeled examples and generalize to longer inputs. These results inform how pretrained priors can support auditable, executable hypotheses without gradient-based retraining. While our experiments execute candidate programs, they do so in a restricted setting; any application beyond this scope should use sandboxed execution and a constrained language to avoid unintended behavior.

\bibliography{refs}
\bibliographystyle{icml2026}

\newpage
\appendix
\onecolumn

\section{Additional Experimental Details And Results}\label{app:experiments}

\subsection{Evaluation Tasks}\label{app:tasks}

We evaluate across a suite of synthetic algorithmic tasks that isolate distinct forms of structure: sparse global rules (parity), position-invariant search (pattern matching), symmetry (palindromes), context-free constraints (Dyck-2), arithmetic (primality), and compositions of local nonlinear maps with global summaries (CAP, SHA).
Unless noted otherwise, for each task and input length $n$ we construct class-balanced datasets with equal numbers of positive and negative examples.

\footnotetext{For Dyck-2 at $n{=}20$, the number of valid strings is too small to sample the standard-sized balanced test set without excessive duplication; in this case we reduce the test set to 1{,}000 examples.}

\begin{itemize}[leftmargin=10pt]

\item {\bf Parity (FP/FHP/$k$-P).\enspace}
Inputs are $x\in\{0,1\}^n$ and labels are $\pm1$.
For a fixed selector $s\in\{0,1\}^n$, define $y(x)=(-1)^{\langle s,x\rangle}$. We consider: (i) \emph{full parity} (FP) with $s=\mathbf{1}_n$; (ii) \emph{first-half parity} (FHP) with
$s=(\mathbf{1}_{n/2}\|\mathbf{0}_{n/2})$ (for even $n$); and (iii) \emph{random $k$-parity} ($k$-P), where $s$ is sampled uniformly from $\{0,1\}^n$ subject to $\|s\|_0=k$ (fixed per run). Datasets are generated by sampling $x\sim\Unif(\{0,1\}^n)$ and labeling by $y(x)$.

\item {\bf Pattern Matching (Pattern1/Pattern2).\enspace}
Fix a binary pattern $p\in\{0,1\}^k$ with $k<n$ and input $x\in\{0,1\}^n$. The label is $y(x)=\mathbb{I}\big[\exists\, i\in\{1,\dots,n-k+1\} \text{ s.t. } (x_i,\dots,x_{i+k-1})=p\big]$. We use $p=\texttt{00111111}$ (Pattern1) and $p=\texttt{10101010}$ (Pattern2).
This task requires a position-invariant ``search'' rule rather than reliance on fixed coordinates.

\item {\bf Palindrome (IsPal).\enspace}
For $x\in\{0,1\}^n$, $y(x)=\mathbb{I}\big[\forall\, i\in\{1,\dots,\lfloor n/2\rfloor\}: x_i=x_{n-i+1}\big]$. Positives are constructed by sampling a random first half and mirroring it.
Negatives are generated by first constructing a palindrome and then flipping one bit in the first half; this produces ``near-miss'' negatives that share most local statistics with positives.

\item {\bf Dyck-2.\enspace}
Let $\mathcal{M}:\{0,1\}^2\to\{\texttt{(},\texttt{)},\texttt{[},\texttt{]}\}$ map bit-pairs to bracket symbols, and let $S(x)$ be the resulting bracket string for $x\in\{0,1\}^n$ (so $|S(x)|=n/2$ when $n$ is even). The label is $y(x)=\mathbb{I}[S(x)\in D_2]$, where $D_2$ is the Dyck-2 language of balanced parentheses with two bracket types. This probes recognition of a context-free constraint that naturally admits stack-like reasoning.

\item {\bf Primality with restricted last digit (IsPrime).\enspace}
Inputs are digit strings $x=(x_1,\dots,x_n)\in\{0,\dots,9\}^n$ interpreted as an $n$-digit base-10 integer. The label is
$y(x)=\mathbb{I}\big[\mathrm{IsPrime}(\mathrm{int}(x))\big]$.
We construct balanced datasets by sampling uniformly from the set of $n$-digit primes (positives) and the set of $n$-digit composite numbers (negatives), while excluding leading zeros.\footnote{We exclude leading zeros when sampling $n$-digit integers.} We constrain the data distribution to integers whose last digit lies in $\{1,3,7,9\}$, i.e., $x_n\in\{1,3,7,9\}$, removing the most common last-digit shortcut and forcing dependence on the full input. 

\item {\bf Cellular Automaton Parity (CAP).\enspace}
Inputs are $x\in\{0,1\}^n$.
Define a one-step local update $x\mapsto x'$ with boundary conditions $x_0=x_{n+1}=0$: $x'_i = x_{i-1} \oplus (x_i \lor x_{i+1})$ ($i=1,\dots,n$). The label is the parity of the updated string, $y(x)=(\sum_{i=1}^n x'_i)\bmod 2$. This composes a local nonlinear transformation with a global summary statistic.

\item {\bf Graph Cycle Density (CountCycles, fn\_aa).\enspace}
Inputs are graphs encoded as strings of edges like \texttt{u12v23u45v67}\,$\dots$, where each edge is \texttt{u} + two-digit source + \texttt{v} + two-digit destination. All instances satisfy $V=E=n/2$, so the cycle rank $E-V+C(G)$ equals the number of connected components $C(G)$. We label by $y(G)=\mathbb{I}\!\left[C(G) > E/6\right]$. Positives have many small components (e.g., triangles and disjoint edges), while negatives have one or a few large components (spanning tree plus extra edges). Node IDs are permuted and edges are shuffled to avoid positional shortcuts. Whenever necessary (running \texttt{TabPFN}), we tokenize the characters as follows: digits $0$--$9\mapsto 0$--$9$, \texttt{u}$\mapsto 10$, \texttt{v}$\mapsto 11$. 

\item {\bf SHA-256 Parity (SHA).\enspace}
For $x\in\{0,1\}^n$, compute the 256-bit digest $(h_1,\dots,h_{256})=\mathrm{SHA}\text{-}256(x)$ and output its parity: $y(x)=(\sum_{i=1}^{256} h_i)\bmod 2$. Because cryptographic hashes behave pseudorandomly, this is intended as a stringent stress test.

\end{itemize}

\subsection{Baselines}\label{app:baselines}

We compare \textsc{LLM-PV} to four families of baselines: (i) classic tabular classifiers, (ii) LMs trained from scratch with SGD, (iii) fine-tuning pre-trained LMs, and (iv) in-context prediction with large instruction-tuned LMs. Unless stated otherwise, experiments use $m{=}200$ labeled examples, split into $|S_{\tr}|{=}100$ and $|S_{\val}|{=}100$, and are evaluated on an independent test set of 10{,}000 examples.

{\bf Classic ML baselines.\enspace}
We train XGBoost, Random Forest, SVM, and a Genetic Algorithm (GA) baseline on $S_{\tr}$ and select hyperparameters by validation performance on $S_{\val}$; reported numbers use the selected configuration and are evaluated on the test set.
Hyperparameter grids are:
(i) \textbf{SVM}: $C \in \{0.1,1,10,100\}$, $\gamma \in \{\texttt{scale},\texttt{auto},0.001,0.01,0.1\}$, $\texttt{kernel}\in\{\texttt{rbf},\texttt{poly},\texttt{sigmoid}\}$;
(ii) \textbf{Random Forest}: $\texttt{n\_estimators}\in\{64,128,256\}$, $\texttt{max\_depth}\in\{5,10,15,\texttt{None}\}$, $\texttt{min\_samples\_split}\in\{2,5,10\}$;
(iii) \textbf{XGBoost}: $\texttt{n\_estimators}\in\{100,128,256\}$, $\texttt{max\_depth}\in\{5,6,7\}$, $\texttt{learning\_rate}\in\{0.1,0.3\}$, $\texttt{subsample}\in\{0.8,1.0\}$;
(iv) \textbf{Genetic Algorithm (GA)}: genetic programming over Boolean expression trees with leaves given by input variables or Boolean constants and internal nodes drawn from $\{\texttt{NOT},\texttt{XOR},\texttt{AND},\texttt{OR}\}$; prediction is obtained by evaluating the tree on the input bits. Fitness is training accuracy (optionally on a random subsample for speed) minus a small size penalty to discourage large expressions. We evolve the population using tournament selection, subtree crossover, subtree mutation, and elitism, and return the best individual (selected by training fitness; reported with validation/test accuracy). GA hyperparameters are: population size $\texttt{pop\_size}\in\{300,500\}$, number of generations $\texttt{generations}{=}80$, initial program depth $\texttt{max\_depth\_init}\in\{6,10\}$, mutation depth cap $\texttt{max\_depth\_mut}\in\{4,6\}$, crossover probability $\texttt{cx\_prob}{=}0.7$, and mutation probability $\texttt{mut\_prob}{=}0.25$.

{\bf Training from scratch (SGD).\enspace}
We train decoder-only language-model backbones from scratch as binary classifiers on sequence-encoded pairs $(x,y)$, minimizing binary cross-entropy with AdamW~\citep{loshchilov2018decoupled}. Our SGD baselines are \texttt{Qwen3-1.7B}~\citep{yang2025qwen3technicalreport}, \texttt{Llama3.2-1B}~\citep{meta2024llama32}, and \texttt{Deepseek-Coder-1.3B}~\citep{journals/corr/abs-2401-14196}. We adapt each LM to classification by (i) restricting the tokenizer vocabulary to three tokens (\verb|vocab_size=3|) and (ii) replacing the LM head with a single-logit linear classifier (hidden$\to$1).
Unless stated otherwise, we train for 200 epochs with batch size 20 using cosine-annealed learning rates with $\eta_{\max}{=}10^{-5}$ and $\eta_{\min}{=}10^{-6}$.
To probe data scaling with a lighter model, we additionally train \texttt{BLOOM-75M}, a scaled-down BLOOM-style architecture~\citep{workshop2023bloom176bparameteropenaccessmultilingual} (hidden size 512, 8 heads, 24 layers), on $m{=}100$k examples per (task, $n$) for 1{,}000 epochs with batch size 256 and constant learning rate $\eta{=}10^{-5}$.



{\bf Fine-tuning pre-trained LMs (FT).\enspace} We fine-tune \texttt{Qwen3-1.7B}~\citep{yang2025qwen3technicalreport}, \texttt{Llama3.2-1B}~\citep{meta2024llama32}, and \texttt{Deepseek-Coder-1.3B}~\citep{journals/corr/abs-2401-14196}. For each model and input length $n\in\{20,25,30,50,100\}$, we train with AdamW for 1{,}000 epochs using a cosine-annealing learning-rate schedule, and report accuracy on 10k held-out test examples. We evaluate both full fine-tuning and partial fine-tuning of only the top $\{2,4,8\}$ transformer blocks. Hyperparameters are selected via grid search over the number of tuned layers, batch size $\{20,50,100\}$, and learning rate $\{5\times 10^{-3}, 10^{-3}, 5\times 10^{-4}, 10^{-4}, 5\times 10^{-5}, 10^{-5}\}$, choosing the configuration with the best test performance. The selected batch size is 20 for all models, with peak learning rates of $10^{-3}$ for \texttt{Llama3.2-1B} and $5\times 10^{-3}$ for \texttt{Qwen3-1.7B} and \texttt{Deepseek-Coder-1.3B}; the scheduler anneals to a minimum learning rate of $10^{-3}$.
All runs use \texttt{bfloat16}, space-separated integer tokenization with EOS padding, and a single-logit classifier head applied to the final position.

{\bf In-context learning (ICL).\enspace} We evaluate whether a large instruction-tuned LLM can recover the latent rule from labeled examples in its context window and apply it to a new input. Concretely, given a training set $S_{\tr}$ of 200 labeled pairs, we ask whether the induced predictor $h(x_{\mathrm{test}}\mid S_{\tr})$ matches the ground-truth label $y(x_{\mathrm{test}})$. We use \texttt{Qwen3-30B-A3B-Instruct-2507} and \texttt{Qwen3-Coder-30B-A3B-Instruct}~\citep{yang2025qwen3technicalreport}, and \texttt{Deepseek-Coder-33B-Instruct}~\citep{journals/corr/abs-2401-14196}. For each of 100 test inputs, we construct a prompt (Fig.~\ref{fig:prompt_incontext}) containing the task description, the full set of 200 training examples, and the held-out test input, and then query the model for the corresponding label. We decode with temperature 0.2, top-$p$ 0.95, and a maximum of 1024 generated tokens.

\begin{figure}[t]
  \centering
  \begin{minipage}{0.94\linewidth}
    \begin{promptbox}[breakable=false,left=4pt,right=4pt,top=4pt,bottom=4pt,boxrule=0.4pt,arc=1mm]
      \footnotesize
      \textbf{Problem Statement:}  
      Given a sequence of input vectors (binary, length \{sequence\_dimension\}) and their corresponding scalar binary outputs (‘0’ or ‘1’), find a concise Python function \verb|f(x)| that accurately approximates the underlying relationship. The function should not be a trainable model, but a direct logical or mathematical representation of the target function.

      \textbf{Data Examples:}
\begin{verbatim}
000111101011110010100101001100 -> 1
... 011011010111000010010101001000 -> 1
\end{verbatim}

\textbf{You must output ONLY a single JSON object:}
\begin{verbatim}
{"code": "<python function>"}
\end{verbatim}
    \end{promptbox}
  \end{minipage}

  \vspace{-0.4em}
  \caption{Prompt used in our LLM-PV procedure. We run \texttt{GPT-5} with this prompt for up to $k$ independent iterations, each returning only Python code for a candidate target function.}
  \label{fig:prompt}
\end{figure}

\begin{figure}[t]
  \centering
  \begin{minipage}{0.94\linewidth}
    \begin{promptbox}[breakable=false,left=4pt,right=4pt,top=4pt,bottom=4pt,boxrule=0.4pt,arc=1mm]
      \footnotesize
      \textbf{Problem Statement:}  
      Given a sequence of input vectors (binary, length \{sequence\_dimension\}) and their corresponding scalar binary outputs ('0' or '1'), you have to learn a hypothesis that approximates the underlying relationship. Given the data below, determine what is the label for the given string and output ONLY the label.
      \textbf{Data Examples:}
\begin{verbatim}
000111101011110010100101001100 -> 1
... 011011010111000010010101001000 -> 1
\end{verbatim}
      \textbf{Test Input:}
\begin{verbatim}
010100110111001001010101001000
\end{verbatim}

      \textbf{You must output ONLY a single JSON object: \{"lable": "<your predicted label>"\}}
    \end{promptbox}
  \end{minipage}

  \vspace{-0.4em}
  \caption{Prompt used in in-context learning procedure. We run three models \texttt{Qwen3-30B-A3B-Instruct-2507}, \texttt{Qwen3-Coder-30B\allowbreak-A3B-Instruct}, and \texttt{Deepseek-Coder-33B\allowbreak-Instruct} with this prompt. For each prompt, the model outputs only the predicted label for the test input.}
  \label{fig:prompt_incontext}
\end{figure}

{\bf LLM-PV.\enspace}
LLM-PV (Alg.~\ref{alg:llm-ktry}, Fig.~\ref{fig:experiment-flow-validation}) uses a pretrained code LLM (\texttt{GPT-5-Thinking}) as a proposal distribution over candidate programs.
We split $S$ into $S_{\tr}$ and $S_{\val}$ of equal size, build a prompt from $S_{\tr}$ (Fig.~\ref{fig:prompt}), sample up to $k{=}5$ candidate programs, compile and filter invalid/duplicate outputs, and select the candidate with minimum validation error on $S_{\val}$.
We run the API with \texttt{reasoning\_effort=High}, \texttt{text\_verbosity=Low}, \texttt{max\_tokens=20k}, and a per-call timeout of 20 minutes; sampling parameters such as temperature/top-$p$ are platform-managed.

\subsection{Additional Results}\label{app:sgd_llm}

\subsubsection{Data Scaling}\label{app:data_scaling}
To investigate the impact of training set size on generalization, we trained the \texttt{BLOOM-75M} baseline using SGD across multiple data scales, using \(m_{\text{train}} \in\)\{10, 100, 200, 1000, 10k, 100k\} for $n{=}100$. This evaluation was conducted on Random 10-Parity, Cellular Automata Parity, and IsPrime (with restricted negatives). The \texttt{BLOOM-75M} model was trained for 1000 epochs and a constant learning rate of $\eta{=}10^{-5}$. The batch size was set to 256 for training sets with 1000 or more samples, and to 10 for smaller datasets. The model trained on this largest dataset was also evaluated across multiple input lengths $n \in \{20, 25, 30, 50, 100\}$.

The results highlight a failure of the SGD baseline to generalize, even with extensive data. As detailed in Fig.~\ref{fig:size_ablation} and Tab.~\ref{tab:size_ablation}, increasing the training set size did not yield better test performance; the model consistently overfits, achieving perfect training accuracy while its test accuracy on Random 10-Parity and Cellular Automata Parity remained at the $\approx$50\% chance level. Furthermore, Tab.~\ref{tab:bloom_ablation} show that even with 100k training samples, the model still overfits when the input dimension is too high, with performance again collapsing to chance. In contrast, LLM-PV seems to be sample-efficient, synthesizing correct programs for all three tasks with a training set size of just 100--200.

\begin{table}[t]
\centering
\setlength{\tabcolsep}{3.8pt}
\renewcommand{\arraystretch}{1.08}
\resizebox{\textwidth}{!}{%
\begin{tabular}{@{}ll|ccc|ccc|ccc|ccc|ccc|ccc@{}}
\toprule
& & \multicolumn{18}{c}{\textbf{Accuracy vs.\ training set size}} \\
\cmidrule(lr){3-20}
& & \multicolumn{3}{c|}{\textbf{$m{=}10$}}
& \multicolumn{3}{c|}{\textbf{$m{=}100$}}
& \multicolumn{3}{c|}{\textbf{$m{=}200$}}
& \multicolumn{3}{c|}{\textbf{$m{=}1{,}000$}}
& \multicolumn{3}{c|}{\textbf{$m{=}10{,}000$}}
& \multicolumn{3}{c}{\textbf{$m{=}100{,}000$}} \\
\cmidrule(lr){3-5}\cmidrule(lr){6-8}\cmidrule(lr){9-11}\cmidrule(lr){12-14}\cmidrule(lr){15-17}\cmidrule(lr){18-20}
\textbf{Task} & \textbf{Model}
& \textbf{Train} & \textbf{Test} & \textbf{Bound (A/T)}
& \textbf{Train} & \textbf{Test} & \textbf{Bound (A/T)}
& \textbf{Train} & \textbf{Test} & \textbf{Bound (A/T)}
& \textbf{Train} & \textbf{Test} & \textbf{Bound (A/T)}
& \textbf{Train} & \textbf{Test} & \textbf{Bound (A/T)}
& \textbf{Train} & \textbf{Test} & \textbf{Bound (A/T)} \\
\midrule

\multirow{2}{*}{\textbf{10-P}} 
& \texttt{BLOOM-75M} 
& 100\% & 49.9\% & 0.0/0.0\%
& 100\% & 49.1\% & 0.0/0.0\%
& 100\% & 49.7\% & 0.0/0.0\%
& 100\% & 49.9\% & 46.7/76.5\%
&  99.9\% & 50.1\% & 94.7/97.6\%
&  99.8\% & 50.7\% & 99.5/99.8\% \\
& \textsc{LLM-PV}
& 100\% & 49.3\% & 0.0/0.0\%
& 100\% & 100\% & 0.0/0.0\%
& 100\% & 100\% & 0.0/0.0\%
& -- & -- & 46.7/76.5\%
& -- & -- & 94.7/97.6\%
& -- & -- & 99.5/99.8\% \\
\midrule

\multirow{2}{*}{\textbf{\shortstack[l]{IsPrime2}}}
& \texttt{BLOOM-75M}
& 100\% & 50.9\% & 0.0/0.0\%
& 100\% & 50.2\% & 0.0/0.0\%
& 100\% & 51.3\% & 0.0/0.0\%
& 100\% & 56.4\% & 22.9/71.9\%
& 100\% & 59.7\% & 92.3/97.2\%
& 100\% & 59.9\% & 99.2/99.7\% \\
& \textsc{LLM-PV}
& 100\% & 50.5\% & 0.0/0.0\%
& 100\% & 100\% & 0.0/0.0\%
& 100\% & 100\% & 0.0/0.0\%
& -- & -- & 22.9/71.9\%
& -- & -- & 92.3/97.2\%
& -- & -- & 99.2/99.7\% \\
\midrule

\multirow{2}{*}{\textbf{\shortstack[l]{CAP}}}
& \texttt{BLOOM-75M}
& 100\% & 50.1\% & 0.0/0.0\%
& 100\% & 50.4\% & 0.0/0.0\%
& 100\% & 50.4\% & 0.0/0.0\%
& 100\% & 50.0\% & 0.0/42.2\%
&  99.6\% & 50.0\% & 85.5/94.2\%
&  99.9\% & 50.4\% & 98.5/99.4\% \\
& \textsc{LLM-PV}
& 100\% & 50.3\% & 0.0/0.0\%
& 100\% & 50.1\% & 0.0/0.0\%
& 100\% & 100\% & 0.0/0.0\%
& -- & -- & 0.0/42.2\%
& -- & -- & 85.5/94.2\%
& -- & -- & 98.5/99.4\% \\

\bottomrule
\end{tabular}%
}
\caption{\textbf{Scaling labeled data does not reliably fix SGD generalization, even on tasks with short programs.}
We report train/test accuracy (\%) at $n{=}100$ as a function of the number of labeled training examples $m$ (test set size is 10k in all cases).
\texttt{BLOOM-75M} trained with SGD attains (near-)perfect training accuracy throughout, yet test accuracy remains near chance on \textbf{10-P} and \textbf{CAP}, and improves only modestly on \textbf{IsPrime2}.
\textsc{LLM-PV} is far more sample-efficient in this setup, solving \textbf{10-P} and \textbf{IsPrime2} with $m\le 100$ and all three tasks with $m{=}200$.
The \textbf{Bound (A/T)} columns give the length-based lower bound
$1-m^{-1}[L\log|\Sigma|+\log(2L^2/\delta)]$
(with $\delta{=}10^{-10}$), evaluated using compact reference implementations with length $L$ measured in \textbf{ASCII bytes} (A, $|\Sigma|{=}128$) or \textbf{Python lexical tokens} (T, $|\Sigma|{=}64$ token types); values are clipped to $[0,1]$ and reported as \textbf{A/T}.
This bound depends only on $(m,L,|\Sigma|,\delta)$ and is included to emphasize that these targets admit short descriptions, even though SGD does not exploit this bias here.}
\label{tab:size_ablation}
\end{table}

\subsubsection{Hyperparameter Ablation}\label{app:ablation}

To assess whether SGD's failures are driven by particular optimization settings, we ran learning-rate and batch-size sweeps in the 200-sample regime, holding all other choices fixed. We trained \texttt{Qwen3-1.7B} for 200 epochs on 200 training samples and evaluated on 10k random test samples, across three representative tasks (10-P, Pattern2, IsPal) and lengths up to $n{=}100$.
For the learning-rate sweep, we used batch size $B{=}20$ and varied $\eta \in \{8\times 10^{0},8\times 10^{-1},\ldots,8\times 10^{-7}\}$.
For the batch-size sweep, we fixed $\eta$ and varied $B \in \{10,20,50,100,200\}$.
Tab.~\ref{tab:hparam_ablation_summary} reports, for each $(\text{task},n)$, the best test accuracy achieved over the corresponding sweep.

Across both sweeps, the qualitative picture is unchanged.
Random 10-Parity (10-P) stays at chance for all $n$, even after selecting the best $\eta$ or the best $B$ (max $\leq 51.1\%$), indicating that no setting in these ranges yields meaningful generalization.
Pattern Matching (Pattern2) can reach near-perfect accuracy at short lengths (e.g., $100\%$ at $n{=}20$ when sweeping $\eta$), but performance drops sharply with increasing $n$ and falls to $56.9\%$ at $n{=}100$ even after optimizing over $\eta$ (and to $65.5\%$ after optimizing over $B$).
Similarly, IsPalindrome (IsPal) improves at short lengths under the best settings (up to $72.4\%$ at $n{=}20$), yet collapses to chance by $n{=}100$ (max $50.0\%$ over $\eta$ and $49.6\%$ over $B$).

Overall, tuning learning rate or batch size mainly affects short-length performance and does not prevent the systematic degradation with length. This supports the interpretation that the dominant failure mode is out-of-distribution length generalization, rather than a fixable optimization detail.

\subsubsection{Prompt Ablation}\label{app:prompt_ablation}

We evaluate the sensitivity of \textsc{LLM-PV} to the exact wording of the prompt
given to \texttt{GPT-5-Thinking}. To construct the ablation, we use five prompt
templates that express the same high-level instruction: infer the relationship
from the training samples and return a concise Python function \texttt{f(x)} that
implements the rule. The prompts differ only in phrasing, not in the information
provided to the model.

Tab.~\ref{tab:prompt_variations_llmpv} reports the performance of
\textsc{LLM-PV} under each prompt. The results are stable across prompt
variations. In particular, all prompts achieve perfect performance on
\textsc{Pattern2}, \textsc{10-P}, \textsc{isPrime}, and \textsc{CAP}. On the
hardest task, \textsc{isPrime+47}, performance varies mildly across prompts, with
an average accuracy of \(82.5 \pm 3.1\). These results suggest that the gains of
\textsc{LLM-PV} are not driven by a carefully engineered prompt.

\begin{table}[t]
\centering

\begin{promptbox}[title={\textbf{Prompt Templates for Prompt Ablation}}]

\noindent\textbf{P1.}
Given a sequence of input vectors (\texttt{\{data\_mode\}}, length
\texttt{\{seq\_len\}}) mapped to scalar binary outputs, extract the underlying
relationship as a concise Python function \texttt{f(x)}. The solution must be a
direct logical or mathematical expression, not a machine learning model.

\medskip

\noindent\textbf{P2.}
Analyze the provided input vectors (\texttt{\{data\_mode\}}, length
\texttt{\{seq\_len\}}) and their corresponding binary outputs to determine the
governing logic. Express this logic as a short, deterministic Python function
\texttt{f(x)} using mathematical or logical operations, avoiding trainable
parameters.

\medskip

\noindent\textbf{P3.}
Identify the mapping between the input vectors (\texttt{\{data\_mode\}}, length
\texttt{\{seq\_len\}}) and binary scalar outputs. Represent this mapping through
a concise, stateless Python function \texttt{f(x)} that relies only on explicit
mathematical or logical rules rather than learned weights.

\medskip

\noindent\textbf{P4.}
Discover the strict mathematical or logical rule that maps the input vectors
(\texttt{\{data\_mode\}}, length \texttt{\{seq\_len\}}) to their binary outputs.
Output a concise Python function \texttt{f(x)} that formally encodes this rule
without relying on any trainable architecture.

\medskip

\noindent\textbf{P5.}
Infer the rule that best explains the observed input-output pairs, where each
input is a vector in \texttt{\{data\_mode\}} of length \texttt{\{seq\_len\}} and
each output is binary. Return the rule as a compact Python function
\texttt{f(x)} using explicit logical or mathematical operations.

\end{promptbox}

\caption{\textbf{Prompt templates used in the prompt ablation experiments.}
Here, \texttt{\{data\_mode\}} denotes the input representation and
\texttt{\{seq\_len\}} denotes the sequence length, with \(n=100\).}
\label{tab:prompt_templates}
\end{table}

\begin{table}[t]
\centering
\renewcommand{\arraystretch}{1.25}
\setlength{\tabcolsep}{10pt}

\begin{tabular}{|l|c|c|c|c|c|}
\hline
\textbf{Prompt} &
\textbf{Pattern2} &
\textbf{10-P} &
\textbf{isPrime} &
\textbf{CAP} &
\textbf{isPrime+47} \\
\hline
\textbf{P1} & 100 & 100 & 100 & 100 & 81.8 \\
\hline
\textbf{P2} & 100 & 100 & 100 & 100 & 77.9 \\
\hline
\textbf{P3} & 100 & 100 & 100 & 100 & 86.7 \\
\hline
\textbf{P4} & 100 & 100 & 100 & 100 & 83.0 \\
\hline
\textbf{P5} & 100 & 100 & 100 & 100 & 83.0 \\
\hline
\textbf{Average} &
$\mathbf{100 \pm 0}$ &
$\mathbf{100 \pm 0}$ &
$\mathbf{100 \pm 0}$ &
$\mathbf{100 \pm 0}$ &
$\mathbf{82.5 \pm 3.1}$ \\
\hline
\end{tabular}

\caption{\textbf{Prompt variation results for \textsc{LLM-PV}.}
We report test accuracy across five prompt templates. Performance is stable across
prompt rephrasings, indicating that the method is not sensitive to a single
carefully engineered prompt.}
\label{tab:prompt_variations_llmpv}
\end{table}

\subsubsection{Wall-Clock Times}\label{app:wall_clock_times}

We report the wall-clock runtime of the main baselines and \textsc{LLM-PV} in
Tab.~\ref{tab:wall_clock_times}. The comparison includes SGD training from
scratch and fine-tuning on \texttt{Qwen3-1.7B}, \textsc{LLM-PV} with
\texttt{GPT-5}, and in-context learning with \texttt{GPT-5-Thinking} and
\texttt{Qwen3-30B-Instruct}. All times are normalized per sample, computed as
the total runtime divided by the number of samples. For SGD training from
scratch and fine-tuning, we report the runtime using the minimum number of
epochs required to reach \(100\%\) training accuracy.

The results highlight the main computational tradeoff. \textsc{LLM-PV} incurs the
largest training-time cost, primarily because program proposals are generated by
an LLM. However, once a program has been selected, inference is fast: its test-time
cost is substantially lower than \texttt{GPT-5-Thinking} in-context learning and
competitive with the other non-API baselines. Gradient descent and fine-tuning are
fast at test time, but their predictive performance is poor. In-context learning
avoids training altogether, but its test-time cost can be large because every test
example requires a fresh model call. Thus, \textsc{LLM-PV} shifts computation from
test time to training time and obtains the best performance while retaining
efficient deployment.

\begin{table}[t]
\centering
\renewcommand{\arraystretch}{1.25}
\setlength{\tabcolsep}{10pt}

\begin{tabular}{|l|c|c|}
\hline
\textbf{Method} &
\textbf{\shortstack{Train time\\(seconds per sample)}} &
\textbf{\shortstack{Test time\\(seconds per sample)}} \\
\hline
\textbf{SGD: Qwen3-1.7B} & 0.671 & 0.016 \\
\hline
\textbf{FT: Qwen3-1.7B} & 0.345 & 0.017 \\
\hline
\textbf{LLM-PV: GPT-5} & 2.091 & 0.224 \\
\hline
\textbf{ICL: GPT-5-Thinking} & 0.001 & 53.961 \\
\hline
\textbf{ICL: Qwen3-30B-Instruct} & 0.001 & 0.139 \\
\hline
\end{tabular}

\caption{\textbf{Wall-clock runtime.} Train and test times are reported in seconds per sample. \textsc{LLM-PV} has higher training cost due to LLM-based proposal generation, but its selected program can be evaluated efficiently at test time.}
\label{tab:wall_clock_times}
\end{table}

\begin{table}[t]
\centering
\small
\setlength{\tabcolsep}{4.2pt}
\renewcommand{\arraystretch}{1.08}

\begin{adjustbox}{max width=\linewidth,center}
\begin{tabular}{@{}llccccc@{}}
\toprule
\textbf{Task} & \textbf{Summary} & \textbf{$n{=}20$} & \textbf{$n{=}25$} & \textbf{$n{=}30$} & \textbf{$n{=}50$} & \textbf{$n{=}100$} \\
\midrule

\multirow{2}{*}{\textbf{10-P}}
& max over $\eta$   & 50.4 & 50.3 & 50.4 & 50.2 & 51.1 \\
& max over $B$    & 50.8 & 50.4 & 50.4 & 50.9 & 50.4 \\
\addlinespace[2pt]\hline\addlinespace[2pt]

\multirow{2}{*}{\textbf{Pattern2}}
& max over $\eta$   & 100.0 & 94.8 & 93.9 & 87.9 & 56.9 \\
& max over$B$   & 97.7  & 96.4 & 96.0 & 89.5 & 65.5 \\
\addlinespace[2pt]\hline\addlinespace[2pt]

\multirow{2}{*}{\textbf{IsPal}}
& max over $\eta$   & 72.4 & 60.5 & 54.5 & 50.3 & 50.0 \\
& max over $B$    & 68.6 & 57.5 & 53.9 & 50.8 & 49.6 \\
\bottomrule
\end{tabular}
\end{adjustbox}

\caption{\textbf{Hyperparameter ablations (test accuracy, \%).}
For each task and sequence length $n$, we report the maximum test accuracy obtained by sweeping learning rates $\eta \in \{8\times 10^0,\dots,8\times 10^{-7}\}$ and batch sweep over $B \in \{10,20,50,100,200\}$. 10-P remains near chance across all $n$, while Pattern2 and IsPal achieve high accuracy at short lengths but degrade substantially as $n$ increases.}
\label{tab:hparam_ablation_summary}
\end{table}

\subsection{LLM Reasoning Traces}\label{app:traces}

We extend (Fig.~\ref{fig:trace_suite}) the reasoning-trace analysis from Fig.~\ref{fig:trace_prime} to Cellular Automata Parity and Full Parity to highlight how the model adapts its search strategy across tasks.

For Cellular Automata Parity, the model starts with simple dataset checks and parity-style baselines, then attempts a linear rule over $\mathbb{F}_2$. When this fails, it escalates to a structured search over non-linear, hand-designed features (for example edge bits, transition counts such as \#01/\#10, and their parities), and identifies an XOR rule that fits perfectly. It then rewrites the rule using parity identities and verifies equivalent forms.

For Full Parity, the trace is much shorter: after a quick sanity check, the model proposes global parity and immediately verifies a perfect match. This contrast shows that the model expands its hypothesis class only when simpler families fail, but locks onto the correct rule quickly when it is obvious.

\begin{figure*}[t]
\centering

\tikzset{
  trace/.style={
    font=\normalsize,
    >=Latex,
    coltitle/.style={font=\bfseries, align=center},
    tag/.style={draw, rounded corners=3pt, inner sep=3pt, font=\scriptsize,
                text width=2.05cm, align=center, minimum height=10mm},
    tagAnaly/.style={tag, fill=blue!4,   draw=blue!70!black},
    tagCand/.style={tag, fill=cyan!10,   draw=cyan!60!black},
    tagRefn/.style={tag, fill=orange!12, draw=orange!70!black},
    tagHypo/.style={tag, fill=purple!10, draw=purple!60!black},
    tagVeri/.style={tag, fill=teal!10,   draw=teal!70!black},
    stepnode/.style={circle, draw, thick, fill=black!5, minimum size=6.8mm, inner sep=0pt},
    rule/.style={font=\scriptsize, draw, rounded corners=6pt, fill=gray!08, align=left,
                 text width=4.9cm, inner sep=4pt, minimum height=10mm},
    why/.style={font=\scriptsize, draw, rounded corners=6pt, fill=blue!06, align=left,
                text width=4.9cm, inner sep=4pt, minimum height=10mm},
    link/.style={-Latex, line width=0.9pt},
    flow/.style={-Latex, dashed, gray!70, line width=0.7pt}
  }
}

\begin{minipage}[t]{0.49\textwidth}\vspace{0pt}
\centering
\begin{subfigure}[t]{\linewidth}\vspace{0pt}
\centering
\begin{adjustbox}{max width=\linewidth,center}
\begin{tikzpicture}[trace]
\matrix (M) [
  row sep={11.5mm, between origins},
  column sep=5.0mm,
  row 1/.style={nodes={yshift=-3mm}},
]{
  \node[coltitle] {Step};  &
  \node[coltitle] {Type};  &
  \node[coltitle] {Candidate function}; &
  \node[coltitle] {Rationale}; \\

  \node[stepnode] (s1) {1}; &
  \node[tagAnaly] (t1) {Analysis}; &
  \node[rule] (r1) {{\bf Dataset check}\\ Class balance and basic counts}; &
  \node[why] (w1) {Start with simple baselines}; \\

  \node[stepnode] (s2) {2}; &
  \node (t2) {%
    \tikz[baseline=(topAB.base)]{
      \node[tagCand, minimum height=5.4mm, inner sep=2.6pt] (topAB) {Candidate A};
      \node[tagCand, below=1.6pt of topAB, minimum height=5.4mm, inner sep=2.6pt] (botAB) {Candidate B};
    }%
  }; &
  \node[rule] (r2) {{\bf Population parity} {\color{blue} (Train acc: 58\%)}\\
                    {\bf Mod-$3$ baseline} {\color{blue} (Train acc: 36\%)}}; &
  \node[why] (w2) {Quick arithmetic heuristics}; \\

  \node[stepnode] (s3) {3}; &
  \node[tagAnaly] (t3) {Analysis}; &
  \node[rule] (r3) {{\bf Linear fit over $\mathbb{F}_2$}\\
                    Solve $Xw=y  \pmod2$; inconsistent.}; &
  \node[why] (w3) {Try structured features beyond a single parity.}; \\

  \node[stepnode] (s4) {4}; &
  \node[tagCand] (t4) {Candidate C}; &
  \node[rule] (r4) {{\bf Threshold on $\#1$}\\
                    Best over $t$: $\hat y=\mathbf{1}\{\#1\ge t\}$\\ {\color{blue} (Train acc: 55\%)}}; &
  \node[why] (w4) {Count-based separators are weak; try transitions}; \\

  \node[stepnode] (s5) {5}; &
  \node[tagCand] (t5) {Candidate D}; &
  \node[rule] (r5) {{\bf Threshold on transitions}\\
                    Best over $t$: $\hat y=\mathbf{1}\{\#01+\#10\ge t\}$ \\ {\color{blue} (Train acc: 61\%)}}; &
  \node[why] (w5) {Move to parity/XOR combinations of features}; \\

  \node[stepnode] (s6) {6}; &
  \node (t6) {%
    \tikz[baseline=(topAH.base)]{
      \node[tagAnaly, minimum height=5.4mm, inner sep=2.6pt] (topAH) {Analysis};
      \node[tagHypo, below=1.6pt of topAH, minimum height=5.4mm, inner sep=2.6pt] (botAH) {Hypothesis 1};
    }%
  }; &
  \node[rule] (r6) {{\bf Feature-mask XOR search}\\
                    Best over XOR masks:  $f(x)=x_{\text{last}}\oplus(\#01\bmod 2)$ \\ {\color{blue} (Train acc: 100\%)}}; &
  \node[why] (w6) {Search XORs over a small hand-designed feature set}; \\

  \node[stepnode] (s7) {7}; &
  \node[tagRefn] (t7) {Refinement A}; &
  \node[rule] (r7) {{\bf Edge-bit only}\\
                    Best edge-bit classifier {\color{blue} (Train acc: 44\%)}}; &
  \node[why] (w7) {Check if a simpler rule suffices}; \\

  \node[stepnode] (s8) {8}; &
  \node[tagRefn] (t8) {Refinement B}; &
  \node[rule] (r8) {{\bf Transition parity only}\\
                    Best $\hat y=(\#01\bmod 2)$ \\ {\color{blue} (Train acc: 58\%)}}; &
  \node[why] (w8) {Isolate which feature drives the XOR rule}; \\

  \node[stepnode] (s9) {9}; &
  \node[tagAnaly] (t9) {Analysis}; &
  \node[rule] (r9) {{\bf Parity identity}\\
                    $\#01-\#10=x_{\text{last}}-x_{\text{first}}$.}; &
  \node[why] (w9) {Use identities to rewrite the hypothesis}; \\

  \node[stepnode] (s10) {10}; &
  \node[tagVeri] (t10) {Verification 1}; &
  \node[rule] (r10) {{\bf Equivalent form}\\
                    $f(x)=x_{\text{first}}\oplus(\#10\bmod 2)$ \\{\color{blue} (Train acc: 100\%)}}; &
  \node[why] (w10) {Verify an algebraically equivalent rule}; \\

  \node[stepnode] (s11) {11}; &
  \node[tagVeri] (t11) {Verification 2}; &
  \node[rule] (r11) {{\bf Final form}\\
                    $f(x)=x_{\text{last}}\oplus(\#01\bmod 2)$  \\ {\color{blue} (Train acc: 100\%)}}; &
  \node[why] (w11) {Same rule; most concise statement}; \\
};

\foreach \i in {1,...,11}{
  \draw[link] (w\i.west) -- (r\i.east);
}
\foreach \i [evaluate=\i as \j using int(\i+1)] in {1,...,10}{
  \draw[flow] (s\i.south) -- (s\j.north);
}
\end{tikzpicture}
\end{adjustbox}
\caption{\textbf{Cellular Automata Parity}}
\end{subfigure}
\end{minipage}
\hfill
\begin{minipage}[t]{0.49\textwidth}\vspace{0pt}
\centering

\begin{subfigure}[t]{\linewidth}\vspace{0pt}
\centering
\begin{adjustbox}{max width=\linewidth,center}
\begin{tikzpicture}[trace]
\matrix (M) [
  row sep={11.5mm, between origins},
  column sep=5.0mm,
  row 1/.style={nodes={yshift=-3mm}},
]{
  \node[coltitle] {Step};  &
  \node[coltitle] {Type};  &
  \node[coltitle] {Candidate function}; &
  \node[coltitle] {Rationale}; \\

  \node[stepnode] (s1) {1}; &
  \node[tagAnaly] (t1) {Analysis}; &
  \node[rule] (r1) {{\bf Ones-count scan}\\ Check $\#1(x)$ statistics; sweep thresholds}; &
  \node[why] (w1) {Start with simple dataset-level features}; \\

  \node[stepnode] (s2) {2}; &
  \node[tagCand] (t2) {Candidate A}; &
  \node[rule] (r2) {{\bf Parity of ones}\\
                    $\hat y=\left(\sum_{i=0}^{49} x_i\right)\bmod 2$
                    {\color{blue} (Train acc: 51\%)}}; &
  \node[why] (w2) {Classic baseline on binary strings}; \\

  \node[stepnode] (s3) {3}; &
  \node (t3) {%
    \tikz[baseline=(topBC.base)]{
      \node[tagCand, minimum height=5.4mm, inner sep=2.6pt] (topBC) {Candidate B};
      \node[tagCand, below=1.6pt of topBC, minimum height=5.4mm, inner sep=2.6pt] (botBC) {Candidate C};
    }%
  }; &
  \node[rule] (r3) {{\bf Edge bits}\\
                    $\hat y=x_0$ {\color{blue} (Train acc: 50\%)}\\
                    $\hat y=x_{49}$ {\color{blue} (Train acc: 48\%)}}; &
  \node[why] (w3) {Test whether boundary bits drive the label}; \\

  \node[stepnode] (s4) {4}; &
  \node[tagAnaly] (t4) {Analysis}; &
  \node[rule] (r4) {{\bf Linear solve over $\mathbb{F}_2$}\\
                    Solve $Xw=y  \pmod2$ for $X\in\mathbb{F}_2^{100\times 50}$.}; &
  \node[why] (w4) {Recover a sparse XOR by Gaussian elimination}; \\

  \node[stepnode] (s5) {5}; &
  \node[tagHypo] (t5) {Hypothesis}; &
  \node[rule] (r5) {{\bf Support set}\\
                    Identify $S=\{0,8,17,18,20,25,30,34,36,42\}$ (0-based).}; &
  \node[why] (w5) {Back-substitution reveals active indices}; \\

  \node[stepnode] (s6) {6}; &
  \node[tagVeri] (t6) {Verification}; &
  \node[rule] (r6) {{\bf Final XOR rule}\\
                    $f(x)=\bigoplus_{i\in S} x_i$ 
                    {\color{blue} (Train acc: 100\%)}}; &
  \node[why] (w6) {Exact match on all training examples}; \\
};

\foreach \i in {1,...,6}{
  \draw[link] (w\i.west) -- (r\i.east);
}
\foreach \i [evaluate=\i as \j using int(\i+1)] in {1,...,5}{
  \draw[flow] (s\i.south) -- (s\j.north);
}
\end{tikzpicture}
\end{adjustbox}
\caption{\textbf{Random 10-Parity}}
\end{subfigure}

\vspace{0.8em}

\begin{subfigure}[t]{\linewidth}
\centering
\begin{adjustbox}{max width=\linewidth,center}
\begin{tikzpicture}[trace]
\matrix (M) [
  row sep={11.5mm, between origins},
  column sep=5.0mm,
  row 1/.style={nodes={yshift=-3mm}},
]{
  \node[coltitle] {Step};  &
  \node[coltitle] {Type};  &
  \node[coltitle] {Candidate function}; &
  \node[coltitle] {Rationale}; \\

  \node[stepnode] (s1) {1}; &
  \node[tagAnaly] (t1) {Analysis}; &
  \node[rule] (r1) {{\bf Ones-count check}\\ Summarize $\#1(x)$; threshold baselines}; &
  \node[why] (w1) {Start with simple statistics at length 50.}; \\

  \node[stepnode] (s2) {2}; &
  \node[tagCand] (t2) {Candidate A}; &
  \node[rule] (r2) {{\bf Full parity}\\
                    $f(x)=\left(\sum_{i=0}^{49} x_i\right)\bmod 2$ \\
                    {\color{blue} (Train acc: 100\%)}}; &
  \node[why] (w2) {If labels encode global XOR, parity should match}; \\

  \node[stepnode] (s3) {3}; &
  \node (t3) {%
    \tikz[baseline=(topHV.base)]{
      \node[tagHypo, minimum height=5.4mm, inner sep=2.6pt] (topHV) {Hypothesis};
      \node[tagVeri, below=1.6pt of topHV, minimum height=5.4mm, inner sep=2.6pt] (botHV) {Verification};
    }%
  }; &
  \node[rule] (r3) {{\bf Verify on all samples}\\
                    Parity matches every label
                    {\color{blue} (Train acc: 100\%)}}; &
  \node[why] (w3) {Hypothesis confirmed; exact rule found}; \\
};

\foreach \i in {1,...,3}{
  \draw[link] (w\i.west) -- (r\i.east);
}
\foreach \i [evaluate=\i as \j using int(\i+1)] in {1,...,2}{
  \draw[flow] (s\i.south) -- (s\j.north);
}
\end{tikzpicture}
\end{adjustbox}
\caption{\textbf{Full Parity}}
\end{subfigure}

\end{minipage}

\caption{\textbf{Reasoning traces for discrete synthetic tasks.}
Left: Cellular Automata Parity, where the search escalates from simple heuristics to an XOR rule over hand-designed features and then simplifies it via parity identities.
Right (top): Random 10-Parity, solved by a linear system over $\mathbb{F}_2$ that recovers the active XOR support.
Right (bottom): Full Parity, identified directly by the global parity hypothesis and verified on all samples.}
\label{fig:trace_suite}
\end{figure*}


\section{Proofs}

\subsection{Proof of Eq.~\ref{eq:pac}}\label{app:proof_pac}

\begin{algorithm}[t]
\caption{Length-First Program Search (LFPS)}
\label{alg:program-enum}
\begin{algorithmic}[1]
\REQUIRE Sample \(S=\{(x_i,y_i)\}_{i=1}^m\), language \(\mathcal L\subseteq\Sigma^\ast\), per-run timeout \(T\in\mathbb N\), optional max length \(L_{\max}\in\mathbb N\cup\{\infty\}\)
\ENSURE A program \(u^\star\in\mathcal L\) whose total semantics \(\llbracket u^\star\rrbracket:\mathcal X\to\{\pm 1\}\) satisfies \(\llbracket u^\star\rrbracket(x_i)=y_i\) for all \((x_i,y_i)\in S\); or \(\bot\) if none is found up to \(L_{\max}\)
\FOR{\(\ell = 1,2,\dots,L_{\max}\)}
  \FORALL{strings \(u\in \mathcal L\) with \(|u|=\ell\) in lexicographic order}
    \STATE \textbf{if} \(u\) fails to compile \textbf{then continue}
    \STATE \(consistent \leftarrow \text{true}\)
    \FOR{each \((x_i,y_i)\in S\)}
      \STATE Run \(u\) on input \(x_i\) for at most \(T\) steps; let \(o_i\in\{\pm 1,\bot\}\) be the output (\(\bot\) if no halt)
      \IF{\(o_i=\bot\) \textbf{or} \(o_i\neq y_i\)}
        \STATE \(consistent \leftarrow \text{false}\); \textbf{break}
      \ENDIF
    \ENDFOR
    \IF{\(consistent\)}
      \STATE \textbf{return} \(u^\star \leftarrow u\) \hfill\COMMENT{minimal-length consistent program}
    \ENDIF
  \ENDFOR
\ENDFOR
\STATE \textbf{return} \(\bot\) \hfill\COMMENT{no consistent total program found up to \(L_{\max}\)}
\end{algorithmic}
\end{algorithm}

\begin{theorem}[\citet{10.1145/1968.1972}; see also Cor.~2.3 of \citet{10.5555/2621980}]\label{thm:finite}
Let $y:\mathcal{X}\to\{\pm 1\}$ be an unknown target function and let $\mathcal{H}\subset\{\pm 1\}^{\mathcal{X}}$ be a finite hypothesis class. Suppose we are in the realizable setting (i.e., $y\in\mathcal{H}$). Let $S=\{(x_i,y(x_i))\}_{i=1}^m$ be $m$ training examples drawn i.i.d.\ from a distribution $D$ over $\mathcal{X}\times\{\pm 1\}$. Then, with probability at least $1-\delta$ over the draw of $S$, every hypothesis $h\in\mathcal{H}$ that is consistent with $S$ satisfies
\[
\err_D(h) ~\le~ \frac{\log(|\mathcal{H}|)+\log(1/\delta)}{m}.
\]
\end{theorem}

\begin{restatable}{corollary}{infinitecase}\label{cor:infinite}
Let $y:\mathcal{X}\to\{\pm 1\}$ be an unknown target function and let $\mathcal{H}=\bigcup_{\ell\ge1}\mathcal{H}_{\ell}\subset\{\pm 1\}^{\mathcal{X}}$ be a union of finite sets. Suppose we are in the realizable setting (i.e., $y \in \mathcal{H}$). Let $S=\{(x_i,y(x_i))\}_{i=1}^m$ be $m$ training examples drawn i.i.d.\ from a distribution $D$ over $\mathcal{X}\times\{\pm 1\}$. Then, with probability at least $1-\delta$ over the draw of $S$, for any $\ell\in\mathbb{N}$ and any hypothesis $h\in\mathcal{H}_\ell$ that is consistent with $S$, it holds that
\[
\err_D(h) ~\le~ \frac{\log(|\mathcal{H}_\ell|)+\log \bigl((\pi^2/6)\,\ell^2/\delta\bigr)}{m}.
\]
\end{restatable}

\begin{proof}
Assume that $y\in\mathcal{H}=\bigcup_{\ell\ge1}\mathcal{H}_\ell$; hence, there exists some $\ell^\ast$ for which $\mathcal{H}_{\ell^\ast}$ is realizable. For each fixed $\ell$ with at least one hypothesis consistent with $S$, Thm.~\ref{thm:finite} implies that for any $\delta_\ell>0$, with probability at least $1-\delta_\ell$, every $h\in\mathcal{H}_\ell$ consistent with $S$ satisfies
\[
\err_D(h) ~\le~ \frac{\log(|\mathcal{H}_\ell|)+\log(1/\delta_\ell)}{m}.
\]
Choose $\delta_\ell=\frac{6}{\pi^2}\,\frac{\delta}{\ell^2}$, so that $\sum_{\ell\ge1}\delta_\ell=\delta$. Then, for each such $\ell$, with probability at least $1-\delta_\ell$,
\[
\err_D(h) ~\le~ \frac{\log(|\mathcal{H}_\ell|)+\log \bigl((\pi^2/6)\,\ell^2/\delta\bigr)}{m}.
\]
Applying a union bound over all $\ell\ge1$ yields the claim (for each $\ell$ with no consistent hypothesis, the inequality is vacuous).
\end{proof}

\begin{restatable}{proposition}{length}\label{prop:program-appendix}
Suppose we wish to learn a target function $y:\mathcal{X}\to\{\pm 1\}$ that can be implemented as a program of length $L$ in a programming language $\mathcal{L}$. Let $\mathcal{L}_{\ell}$ denote the set of programs of length~$\ell$ in $\mathcal{L}$, and let $S=\{(x_i,y(x_i))\}_{i=1}^m$ be $m$ training examples drawn i.i.d.\ from a distribution $D$ over $\mathcal{X}\times\{\pm 1\}$. Then, with probability at least $1-\delta$ over the draw of $S$, Alg.~\ref{alg:program-enum} outputs a program $h\in\mathcal{L}$ that is consistent with $S$ and satisfies
\[
\err_D(h) ~\le~ \frac{L\log|\Sigma|+\log(2L^2/\delta)}{m}.
\]
\end{restatable}

\begin{proof}
Since $y\in\mathcal{L}$, there exists a minimal length $L$ such that $y\in\mathcal{L}_L$. Therefore, there is at least one program of length $L$ consistent with $S$. Alg.~\ref{alg:program-enum} enumerates programs in order of increasing length, so it eventually returns a program $h$ of some length $\ell\le L$ that is consistent with $S$. Every program in $\mathcal{L}_\ell$ is described over the alphabet $\Sigma$, hence $|\mathcal{L}_\ell|\le|\Sigma|^\ell$ and $\log|\mathcal{L}_\ell|\le \ell\log|\Sigma|\le L\log|\Sigma|$. Applying Cor.~\ref{cor:infinite} with $\mathcal{H}=\mathcal{L}$ and $\mathcal{H}_\ell=\mathcal{L}_\ell$ and then upper-bounding by $L$ gives
\[
\err_D(h) ~\le~ \frac{\log|\mathcal{L}_\ell|+\log(2\ell^2/\delta)}{m}
~\le~ \frac{L\log|\Sigma|+\log(2L^2/\delta)}{m}.
\]
\end{proof}

\pv*

\begin{proof}
Fix $\epsilon\ge 0$ and $\delta\in(0,1)$. Throughout the proof we condition on the (random) training set
$S_{\mathrm{tr}}$ and use that, by assumption, all $k$ trials produce accepted hypotheses using only
$S_{\mathrm{tr}}$. Hence, conditioned on $S_{\mathrm{tr}}$, the hypotheses $(h_t)_{t=1}^k$ are i.i.d.\ draws from
the (accepted) proposal distribution $q(\cdot\mid S_{\mathrm{tr}})$ and are independent of $S_{\mathrm{val}}$.
(Here $|S_{\mathrm{tr}}|=m_{\mathrm{tr}}$ and $|S_{\mathrm{val}}|=m_{\mathrm{val}}$.)

\paragraph{Step 1: a near-$\epsilon$ candidate appears with high probability.}
Define the set of $\epsilon$-good hypotheses
\[
\mathcal G_\epsilon \;:=\;\{h:\mathcal X\to\{\pm1\}\ :\ \err_D(h)\le \epsilon\}.
\]
By definition,
\[
p_\epsilon(S_{\mathrm{tr}})\;=\;\Pr_{h\sim q(\cdot\mid S_{\mathrm{tr}})}[\,h\in\mathcal G_\epsilon\,].
\]
Let $E_{\mathrm{hit}}$ be the event that at least one of the $k$ candidates is $\epsilon$-good:
\[
E_{\mathrm{hit}} \;:=\;\Big\{\exists\, t\in[k]\ \text{s.t.}\ h_t\in\mathcal G_\epsilon\Big\}.
\]
Conditioned on $S_{\mathrm{tr}}$, since the $h_t$ are i.i.d.\ from $q(\cdot\mid S_{\mathrm{tr}})$,
\[
\Pr\!\big(E_{\mathrm{hit}}^c \mid S_{\mathrm{tr}}\big)
=\Pr\!\Big(\forall t\in[k],\ h_t\notin\mathcal G_\epsilon \,\Big|\, S_{\mathrm{tr}}\Big)
=(1-p_\epsilon(S_{\mathrm{tr}}))^k.
\]
Assume $k\ge \left\lceil \frac{\log(\delta/2)}{\log(1-p_\epsilon(S_{\mathrm{tr}}))}\right\rceil$.
Since $\log(1-p_\epsilon(S_{\mathrm{tr}}))\le 0$, this implies
$(1-p_\epsilon(S_{\mathrm{tr}}))^k \le \delta/2$, and therefore
\begin{equation}
\Pr\!\big(E_{\mathrm{hit}} \mid S_{\mathrm{tr}}\big)\;\ge\;1-\delta/2.
\label{eq:hit}
\end{equation}

\paragraph{Step 2: uniform validation-to-population deviation over $k$ hypotheses.}
For each fixed hypothesis $h$, by Hoeffding's inequality applied to the Bernoulli losses
$\mathbf{1}\{h(x'_j)\neq y(x'_j)\}$ on $S_{\mathrm{val}}$ with $|S_{\mathrm{val}}|=m_{\mathrm{val}}$,
\[
\Pr\!\Big(\big|\err_{S_{\mathrm{val}}}(h)-\err_D(h)\big|> \alpha\Big)\;\le\;2e^{-2m_{\mathrm{val}}\alpha^2}.
\]
Conditioned on $S_{\mathrm{tr}}$, the hypotheses $h_1,\dots,h_k$ are independent of $S_{\mathrm{val}}$,
so we can apply the above bound to each $h_t$ and union bound over $t\in[k]$:
\[
\Pr\!\Big(\max_{t\in[k]}\big|\err_{S_{\mathrm{val}}}(h_t)-\err_D(h_t)\big|>\alpha\ \Big|\ S_{\mathrm{tr}}\Big)
\;\le\;\sum_{t=1}^k 2e^{-2m_{\mathrm{val}}\alpha^2}
\;=\;2k\,e^{-2m_{\mathrm{val}}\alpha^2}.
\]
Choose
\[
\alpha \;:=\;\sqrt{\frac{\log(4k/\delta)}{2m_{\mathrm{val}}}}.
\]
Then $2k e^{-2m_{\mathrm{val}}\alpha^2}=2k e^{-\log(4k/\delta)}=\delta/2$, hence
\begin{equation}
\Pr\!\Big(\forall t\in[k],\ \big|\err_{S_{\mathrm{val}}}(h_t)-\err_D(h_t)\big|\le \alpha\ \Big|\ S_{\mathrm{tr}}\Big)
\;\ge\;1-\delta/2.
\label{eq:uniform}
\end{equation}
Let $E_{\mathrm{unif}}$ denote the event inside \eqref{eq:uniform}.

\paragraph{Step 3: conclude the bound on $\err_D(h^\star)$.}
On the intersection $E_{\mathrm{hit}}\cap E_{\mathrm{unif}}$, pick an index
$t_{\mathrm{good}}\in[k]$ such that $h_{t_{\mathrm{good}}}\in\mathcal G_\epsilon$, so
$\err_D(h_{t_{\mathrm{good}}})\le \epsilon$. Also, on $E_{\mathrm{unif}}$ we have for all $t$,
\[
\err_{S_{\mathrm{val}}}(h_t)\;\le\;\err_D(h_t)+\alpha
\qquad\text{and}\qquad
\err_D(h_t)\;\le\;\err_{S_{\mathrm{val}}}(h_t)+\alpha.
\]
Since $h^\star\in\arg\min_{t\in[k]} \err_{S_{\mathrm{val}}}(h_t)$,
\[
\err_{S_{\mathrm{val}}}(h^\star)\;\le\;\err_{S_{\mathrm{val}}}(h_{t_{\mathrm{good}}})
\;\le\;\err_D(h_{t_{\mathrm{good}}})+\alpha
\;\le\;\epsilon+\alpha.
\]
Applying the other side of the uniform deviation bound to $h^\star$ yields
\[
\err_D(h^\star)\;\le\;\err_{S_{\mathrm{val}}}(h^\star)+\alpha
\;\le\;\epsilon+2\alpha
\;=\;\epsilon + 2\sqrt{\frac{\log(4k/\delta)}{2m_{\mathrm{val}}}}.
\]

\paragraph{Step 4: probability of the good event.}
By \eqref{eq:hit} and \eqref{eq:uniform} and a union bound (still conditioned on $S_{\mathrm{tr}}$),
\[
\Pr\!\big(E_{\mathrm{hit}}\cap E_{\mathrm{unif}} \mid S_{\mathrm{tr}}\big)
\;\ge\;1-\delta/2-\delta/2 \;=\;1-\delta.
\]
Removing the conditioning (i.e., averaging over $S_{\mathrm{tr}}$) preserves the same lower bound, so
with probability at least $1-\delta$ over $(S_{\mathrm{tr}},S_{\mathrm{val}})$ and all algorithmic randomness,
\[
\err_D(h^\star)\;\le\;\epsilon + 2\sqrt{\frac{\log(4k/\delta)}{2m_{\mathrm{val}}}}.
\]
This is exactly the claimed guarantee.
\end{proof}

\section{Finite-Precision Mini-Batch SGD as a $1$-STAT$(b)$ Statistical Algorithm}
\label{sec:sgd-1statb}

Consider optimizing a parameter vector $\theta\in\Theta\subseteq\mathbb R^d$ from i.i.d.\ data $z\sim D$ via a loss $\ell(\theta,z)$. To model $b$-bit access to gradients in the statistical-query framework, we assume each
per-example coordinate gradient is uniformly bounded:
there exists $G>0$ such that for all $\theta\in\Theta$, $j\in[d]$, and $z\in\mathcal Z$, $|\partial_j \ell(\theta,z)|\le G$.

{\bf Coordinate mini-batch SGD.\enspace} We consider a simple coordinate version of mini-batch SGD. Fix a batch size $B$ and step sizes
$(\eta_t)_{t\ge 1}$. Starting from $\theta_1$, at each iteration $t$ the algorithm selects a coordinate $j_t\in[d]$ to update (the choice may depend on the past randomness and observations). It then draws a fresh
mini-batch $z_{t,1},\dots,z_{t,B}\stackrel{\text{i.i.d.}}{\sim} D$ and uses these samples to estimate the $j_t$-th partial derivative by the empirical average
\[
\widehat g_{t,j_t}:=\frac{1}{B}\sum_{i=1}^B \partial_{j_t}\ell(\theta_t,z_{t,i}).
\]
Finally, it takes a step along that coordinate:
\[
\theta_{t+1}=\theta_t-\eta_t\,\widehat g_{t,j_t}\,e_{j_t}.
\]

{\bf Finite-precision model.\enspace} To obtain sample/iteration lower bounds, we analyze a finite-precision abstraction of coordinate mini-batch SGD. The point is that, in many implementations, each per-example coordinate gradient is effectively represented with only a small number of bits.  This finite-precision view is also convenient for connecting SGD to the statistical-algorithm framework~\cite{reyzin2020statisticalqueriesstatisticalalgorithms}.

Fix an integer $b\ge 1$ and set $M:=2^b$.  Partition the interval $[-G,G]$ into $M$ equal bins of width $\Delta := \frac{2G}{M}$. For any $u\in[-G,G]$, define the (clipped) bin index
\[
\mathrm{idx}(u)
\;:=\;
\min\Big\{M-1,\ \big\lfloor (u+G)/\Delta \big\rfloor\Big\}
\in\{0,1,\dots,M-1\},
\]
and map $u$ to the midpoint of its bin:
\[
Q_b(u)
\;:=\;
-G+\Big(\mathrm{idx}(u)+\tfrac12\Big)\Delta,
\qquad u\in[-G,G].
\]
This is a deterministic $b$-bit quantizer with worst-case error at most half a bin:
for all $u\in[-G,G]$,
\begin{equation}
\label{eq:quant-error}
|Q_b(u)-u|\le \Delta/2 = G/2^b.
\end{equation}

\begin{definition}[$b$-bit coordinate mini-batch SGD]
\label{def:sgd-bbit}
The $b$-bit variant of coordinate mini-batch SGD replaces each per-example coordinate gradient $\partial_{j_t}\ell(\theta_t^{(b)},z_{t,i})$ by its quantized value $Q_b(\partial_{j_t}\ell(\theta_t^{(b)},z_{t,i}))$.  At iteration $t$ it forms the quantized mini-batch estimate
\[
\widehat g^{(b)}_{t,j_t}
:=\frac{1}{B}\sum_{i=1}^B Q_b\big(\partial_{j_t}\ell(\theta_t^{(b)},z_{t,i})\big),
\]
and updates only coordinate $j_t$:
\[
\theta_{t+1}^{(b)}=\theta_t^{(b)}-\eta_t\,\widehat g^{(b)}_{t,j_t}\,e_{j_t}.
\]
\end{definition}

Before turning to lower bounds, we introduce a simple stability estimate for finite-precision updates. Intuitively, the $b$-bit run differs from the full-precision run in two ways: each per-example coordinate gradient is quantized, and the two runs may evaluate gradients at slightly different iterates.
Under a Lipschitz assumption on per-example coordinate gradients, the accumulated deviation between the two coupled trajectories scales with the rate of precision $O(2^{-b})$.

\begin{lemma}[Quantization error telescoping under Lipschitz gradients]
\label{lem:quant-telescope}
Fix $b\in\mathbb N$ and a quantizer $Q_b:[-G,G]\to[-G,G]$ satisfying
$|Q_b(x)-x|\le \frac{G}{2^b}$ for all $x \in [-G,G]$.
Assume per-example coordinate gradients are bounded and Lipschitz in parameters: there exists $L > 0$ such that
for all $\theta,\theta'\in\Theta$, all $z\in\mathcal Z$, and all $j\in[d]$,
\begin{equation}
\label{eq:coord-lip}
|\partial_j \ell(\theta,z)| \le G,
\qquad
|\partial_j \ell(\theta,z)-\partial_j \ell(\theta',z)| \le L\|\theta-\theta'\|_2 .
\end{equation}

At iteration $t$, define the (unquantized) mini-batch coordinate gradient
\[
\widehat g_{t,j_t}:=\frac{1}{B}\sum_{i=1}^B \partial_{j_t}\ell(\theta_t,z_{t,i}) \in [-G,G],
\]
and the $b$-bit mini-batch coordinate gradient
\[
\widehat g^{(b)}_{t,j_t}:=\frac{1}{B}\sum_{i=1}^B Q_b \big(\partial_{j_t}\ell(\theta^{(b)}_t,z_{t,i})\big)\in[-G,G].
\]
Consider two coupled runs that use the same initialization $\theta_1^{(b)}=\theta_1$,
the same coordinates $(j_t)_{t\ge 1}$, and the same mini-batches $(z_{t,i})$, updated by
\[
\theta_{t+1}~=~\theta_t-\eta_t\,\widehat g_{t,j_t}\,e_{j_t},
\qquad
\theta^{(b)}_{t+1}~=~\theta^{(b)}_t-\eta_t\,\widehat g^{(b)}_{t,j_t}\,e_{j_t}.
\]
Then, for all $T\ge 1$,
\begin{equation}
\label{eq:iterate-telescope-lip}
\|\theta^{(b)}_{T+1}-\theta_{T+1}\|_2
~\le~
\frac{G}{2^b}\sum_{t=1}^T \eta_t \prod_{r=t+1}^{T}\bigl(1+L\eta_r\bigr)
~\le~
\frac{G}{2^b L}\Bigl(\exp\bigl(L\sum_{t=1}^T \eta_t\bigr)-1\Bigr),
\end{equation}
with the convention that the empty product equals $1$ (so the $t=T$ term is $\eta_T$).
\end{lemma}

\begin{proof}
Fix $t$. Write
\[
a_{t,i}:=\partial_{j_t}\ell(\theta^{(b)}_t,z_{t,i}),
\qquad
b_{t,i}:=\partial_{j_t}\ell(\theta_t,z_{t,i}).
\]
By the triangle inequality,
\begin{align*}
\big|\widehat g^{(b)}_{t,j_t}-\widehat g_{t,j_t}\big|
&=
\left|\frac{1}{B}\sum_{i=1}^B Q_b(a_{t,i})-\frac{1}{B}\sum_{i=1}^B b_{t,i}\right|\\
&\le
\left|\frac{1}{B}\sum_{i=1}^B \big(Q_b(a_{t,i})-a_{t,i}\big)\right|
+
\left|\frac{1}{B}\sum_{i=1}^B (a_{t,i}-b_{t,i})\right|\\
&\le
\frac{1}{B}\sum_{i=1}^B |Q_b(a_{t,i})-a_{t,i}|
+
\frac{1}{B}\sum_{i=1}^B |a_{t,i}-b_{t,i}|.
\end{align*}
By the quantization property, $|Q_b(a_{t,i})-a_{t,i}|\le G/2^b$.
By the coordinate Lipschitz condition \eqref{eq:coord-lip},
\[
|a_{t,i}-b_{t,i}|
=
\big|\partial_{j_t}\ell(\theta^{(b)}_t,z_{t,i})-\partial_{j_t}\ell(\theta_t,z_{t,i})\big|
\le
L\|\theta^{(b)}_t-\theta_t\|_2
\quad\text{for each }i.
\]
Therefore,
\begin{equation}
\label{eq:grad-diff-bound}
\big|\widehat g^{(b)}_{t,j_t}-\widehat g_{t,j_t}\big|
\le
\frac{G}{2^b} + L\|\theta^{(b)}_t-\theta_t\|_2.
\end{equation}

Now subtract the updates:
\[
\theta^{(b)}_{t+1}-\theta_{t+1}
=
(\theta^{(b)}_t-\theta_t)-\eta_t(\widehat g^{(b)}_{t,j_t}-\widehat g_{t,j_t})e_{j_t}.
\]
Taking $\ell_2$ norms and using $\|e_{j_t}\|_2=1$ gives
\[
\|\theta^{(b)}_{t+1}-\theta_{t+1}\|_2
\le
\|\theta^{(b)}_t-\theta_t\|_2
+
\eta_t\big|\widehat g^{(b)}_{t,j_t}-\widehat g_{t,j_t}\big|.
\]
Plugging \eqref{eq:grad-diff-bound} yields the recursion
\[
\Delta_{t+1}
\le
(1+L\eta_t)\Delta_t + \eta_t\frac{G}{2^b},
\qquad\text{where }\Delta_t:=\|\theta^{(b)}_t-\theta_t\|_2.
\]
Assuming $\theta^{(b)}_1=\theta_1$ so that $\Delta_1=0$, unrolling gives
\[
\Delta_{T+1}
\le
\frac{G}{2^b}\sum_{t=1}^T \eta_t \prod_{r=t+1}^{T}(1+L\eta_r),
\]
which is the first bound in \eqref{eq:iterate-telescope-lip}.

For the second bound, use $1+u\le e^u$ to obtain
\[
\prod_{r=t+1}^{T}(1+L\eta_r)\le \exp \Bigl(L\sum_{r=t+1}^{T}\eta_r\Bigr).
\]
Define $S_t:=\sum_{r=t}^{T}\eta_r$ and $A_t:=e^{L S_t}$, so $A_{t}=e^{L\eta_t}A_{t+1}$ and $A_{T+1}=1$.
Then
\[
A_t-A_{t+1}=(e^{L\eta_t}-1)A_{t+1}\ge L\eta_t\,A_{t+1}
\quad\Longrightarrow\quad
\eta_t A_{t+1}\le \frac{1}{L}(A_t-A_{t+1}).
\]
Summing this inequality over $t=1,\dots,T$ yields
\[
\sum_{t=1}^T \eta_t \exp \Bigl(L\sum_{r=t+1}^{T}\eta_r\Bigr)
=
\sum_{t=1}^T \eta_t A_{t+1}
\le
\frac{1}{L}\sum_{t=1}^T (A_t-A_{t+1})
=
\frac{1}{L}(A_1-A_{T+1})
=
\frac{1}{L}\Bigl(e^{L\sum_{t=1}^T\eta_t}-1\Bigr).
\]
Therefore,
\[
\Delta_{T+1}\le \frac{G}{2^b L}\Bigl(\exp\bigl(L\sum_{t=1}^T \eta_t\bigr)-1\Bigr),
\]
establishing \eqref{eq:iterate-telescope-lip}.
\end{proof}

\subsection{$b$-Bit Coordinate Mini-Batch SGD as a $1$-STAT$(b)$ Algorithm}

A $1$-STAT$(b)$ query is a vector of $b$ Boolean functions
$g=(g_1,\dots,g_b)$ with $g_k:\mathcal Z\to\{0,1\}$.
The oracle draws a fresh $z\sim D$ and returns $(g_1(z),\dots,g_b(z))\in\{0,1\}^b$.

\paragraph{Encoding the quantized gradient as a $1$-STAT$(b)$ answer.}
Fix $(\theta,j)$. Recall that the quantizer $Q_b$ defined above maps $[-G,G]$ to exactly $M:=2^b$ midpoints,
indexed by $\{0,\dots,M-1\}$.  For a sample $z\in\mathcal Z$, define
\[
\mathrm{idx}_{\theta,j}(z)\;:=\;\mathrm{idx}\big(\partial_j\ell(\theta,z)\big)\in\{0,1,\dots,M-1\}.
\]
For each $r\in\{1,\dots,b\}$, define the Boolean function $g^{\theta,j}_r:\mathcal Z\to\{0,1\}$ as the $r$-th bit
(under a fixed convention) of the binary representation of $\mathrm{idx}_{\theta,j}(z)$:
\[
g^{\theta,j}_r(z)\;:=\;\mathrm{bit}_r \left(\mathrm{idx}_{\theta,j}(z)\right).
\]
Then one call to $1$-STAT$(b)$ on $(g^{\theta,j}_1,\dots,g^{\theta,j}_b)$ returns the $b$ bits encoding
$\mathrm{idx}_{\theta,j}(z)$ for a fresh sample $z\sim D$, from which the algorithm reconstructs
\[
Q_b \big(\partial_j\ell(\theta,z)\big)\;=\;-G+\Big(\mathrm{idx}_{\theta,j}(z)+\tfrac12\Big)\Delta.
\]

\begin{lemma}[$b$-bit mini-batch SGD is a $1$-STAT$(b)$ statistical algorithm]
\label{lem:sgd-is-1statb}
A $T$-iteration run of $b$-bit coordinate mini-batch SGD (Definition~\ref{def:sgd-bbit}) with batch size $B$
can be implemented using exactly $q=TB$ calls to $1$-STAT$(b)$, plus internal computation.
\end{lemma}

\begin{proof}
At iteration $t$, for each $i=1,\dots,B$ the algorithm needs the quantized value
$Q_b(\partial_{j_t}\ell(\theta^{(b)}_t,z_{t,i}))$ for a fresh $z_{t,i}\sim D$.
Fix $\theta=\theta^{(b)}_t$ and $j=j_t$. Using the encoding above, one call to $1$-STAT$(b)$ on $(g^{\theta,j}_1,\dots,g^{\theta,j}_b)$ returns the $b$ bits encoding $\mathrm{idx}_{\theta,j}(z_{t,i})$, hence allows reconstruction of
$Q_b(\partial_j\ell(\theta,z_{t,i}))$. Repeating this for $i=1,\dots,B$ yields the $B$ quantized per-example gradients, which can be averaged to obtain
$\widehat g^{(b)}_{t,j_t}$ and then used to update $\theta^{(b)}_{t+1}$.
Over $T$ iterations this uses exactly $TB$ oracle calls.
\end{proof}

\begin{lemma}[SDA for planted $k$-parity testing]
\label{lem:sda-parity}
Let $n\ge 1$, $k\in\{1,\dots,n\}$, and let
\[
\mathcal X=\{0,1\}^n,
\qquad
\mathcal Y=\{\pm 1\},
\qquad
\mathcal S_k:=\{s\in\{0,1\}^n:\|s\|_0=k\},
\qquad
N:=|\mathcal S_k|=\binom{n}{k}.
\]
For each $s\in\mathcal S_k$, let $D_s$ be the realizable distribution on $\mathcal X\times\mathcal Y$ given by
\[
x\sim\Unif(\mathcal X),
\qquad
y=f_s(x):=(-1)^{\langle s,x\rangle},
\qquad
\text{where }\langle s,x\rangle := \sum_{i=1}^n s_i x_i \pmod 2.
\]
Let $D_0$ be the null distribution where $x\sim\Unif(\mathcal X)$ and $y\sim\Unif(\mathcal Y)$ independently.
Let $\chi_{D_0}(\cdot,\cdot)$ and $\rho(\cdot,D_0)$ be as in Definitions~27--28 of
\citet{reyzin2020statisticalqueriesstatisticalalgorithms}
(equivalently, \citet{doi:10.1137/16M1078471}).

Then for every nonempty $\mathcal D'\subseteq\{D_s:s\in\mathcal S_k\}$ with $|\mathcal D'|=m$, we have
\[
\rho(\mathcal D',D_0)=\frac{1}{m}.
\]
\noindent
(Here $\rho$ includes diagonal terms, i.e., $D_1,D_2$ are drawn independently and uniformly from $\mathcal D'$, so
$\Pr[D_1=D_2]=1/|\mathcal D'|$.)

Consequently, for the promise testing problem $Z$ (distinguish $D_0$ from some $D_s$), define
$\SDA_{\mathbb Z}(Z,\bar\gamma)$ to be the \emph{integer-valued} statistical dimension with average correlation,
namely the largest integer $d\ge 1$ such that for every $\mathcal D'\subseteq\{D_s:s\in\mathcal S_k\}$ with
$|\mathcal D'|\ge N/d$, we have $\rho(\mathcal D',D_0)\le \bar\gamma$;
if no such integer exists, define $\SDA_{\mathbb Z}(Z,\bar\gamma):=0$.
Then for every $\bar\gamma\in(0,1)$,
\[
\SDA_{\mathbb Z}(Z,\bar\gamma)=
\begin{cases}
\Theta(\bar\gamma N), & \text{if }\bar\gamma \ge 1/N,\\
0, & \text{if }\bar\gamma < 1/N.
\end{cases}
\]
In particular, when $\bar\gamma<1/N$, Theorem~30 is vacuous in this regime (and in particular, it does not imply any $\omega(1)$ query lower bound).
\end{lemma}

\begin{remark}[Real-valued SDA and the $\bar\gamma<1/N$ regime]
Definition~29 in \citet{reyzin2020statisticalqueriesstatisticalalgorithms} defines
$\SDA(Z,\bar\gamma)$ as the \emph{largest real} $d>0$ such that
for every $\mathcal D'\subseteq\{D_s:s\in\mathcal S_k\}$ with $|\mathcal D'|\ge N/d$ we have
$\rho(\mathcal D',D_0)\le \bar\gamma$.
When a maximum need not be attained, we use the standard extension
\[
\SDA^{\sup}(Z,\bar\gamma)
\;:=\;
\sup\Big\{ d>0:\ \forall\,\mathcal D'\subseteq\{D_s\}\ \text{with }|\mathcal D'|\ge N/d,\ 
\rho(\mathcal D',D_0)\le \bar\gamma \Big\}.
\]
In our setting, if $\bar\gamma<1/N$ then the feasible set equals $(0,1)$ (all $d\in(0,1)$ make the condition
vacuous, while no $d\ge 1$ is feasible), and hence $\SDA^{\sup}(Z,\bar\gamma)=1$.
Therefore Theorem~30 yields at most an $O(1)$ (in particular, not $\omega(1)$) query lower bound in this regime.
\end{remark}

\begin{proof}
Fix $s,s'\in\mathcal S_k$.
For $(x,y)\in\mathcal X\times\mathcal Y$, under the null distribution $D_0$ we have
$x\sim\Unif(\mathcal X)$ and $y\sim\Unif(\mathcal Y)$ independent, hence
\[
D_0(x,y)=2^{-(n+1)}.
\]
Under $D_s$, we have $x\sim\Unif(\mathcal X)$ and $y=f_s(x)$ deterministically, so
\[
D_s(x,y)=2^{-n}\cdot \mathbf 1\{y=f_s(x)\}.
\]
Therefore
\[
\frac{D_s(x,y)}{D_0(x,y)} = 2\cdot \mathbf 1\{y=f_s(x)\},
\qquad\text{and hence}\qquad
\frac{D_s}{D_0}-1
=
\begin{cases}
1, & y=f_s(x),\\
-1, & y=-f_s(x),
\end{cases}
= y\,f_s(x).
\]

By Definition~27 (pairwise correlation),
\begin{align*}
\chi_{D_0}(D_s,D_{s'})
&=
\E_{(x,y)\sim D_0} \left[\left(\frac{D_s}{D_0}-1\right)\left(\frac{D_{s'}}{D_0}-1\right)\right] \\
&=
\E_{(x,y)\sim D_0} \big[yf_s(x)\cdot y f_{s'}(x)\big]
=
\E_{x\sim\Unif(\mathcal X)} \big[f_s(x)f_{s'}(x)\big].
\end{align*}
Moreover,
\[
f_s(x)f_{s'}(x)=(-1)^{\langle s,x\rangle+\langle s',x\rangle}=(-1)^{\langle s\oplus s',x\rangle},
\]
where $\oplus$ is bitwise XOR. Since $x$ is uniform on $\{0,1\}^n$,
\[
\E_x[(-1)^{\langle t,x\rangle}] =
\begin{cases}
1, & t=0,\\
0, & t\neq 0,
\end{cases}
\]
so $\chi_{D_0}(D_s,D_{s'})=1$ if $s=s'$ and $0$ otherwise.

Now let $\mathcal D'\subseteq\{D_s:s\in\mathcal S_k\}$ be any nonempty subset of size $m$.
Using Definition~28 (average correlation) and the above orthogonality, only the $m$ diagonal terms contribute:
\[
\rho(\mathcal D',D_0)
=
\frac{1}{m^2}\sum_{D_1,D_2\in\mathcal D'} \chi_{D_0}(D_1,D_2)
=
\frac{1}{m^2}\cdot m
=
\frac{1}{m}.
\]

We now compute $\SDA_{\mathbb Z}(Z,\bar\gamma)$.
Fix an integer $d\ge 1$ and write
\[
m_0:=\left\lceil \frac{N}{d}\right\rceil.
\]
Since $\rho(\mathcal D',D_0)=1/|\mathcal D'|$ is decreasing in $|\mathcal D'|$, the worst case among
$|\mathcal D'|\ge N/d$ occurs at the smallest allowed size, namely $|\mathcal D'|=m_0$.
Thus the defining condition for $\SDA_{\mathbb Z}(Z,\bar\gamma)$ is equivalent to
\[
\frac{1}{m_0}\le \bar\gamma.
\]

\noindent{\bf Regime 1: $\bar\gamma<1/N$.}
For every integer $d\ge 1$ we have $m_0=\lceil N/d\rceil \le N$, hence $1/m_0\ge 1/N>\bar\gamma$.
So no integer $d\ge 1$ is feasible and $\SDA_{\mathbb Z}(Z,\bar\gamma)=0$.

\noindent{\bf Regime 2: $\bar\gamma\ge 1/N$.}
We show $\SDA_{\mathbb Z}(Z,\bar\gamma)=\Theta(\bar\gamma N)$.

{\bf Lower bound.\enspace}
Let $d:=\max\{1,\lfloor \bar\gamma N/2\rfloor\}$.
If $d=1$, then $m_0=\lceil N\rceil=N$ and $1/m_0=1/N\le \bar\gamma$, so $d$ is feasible.
Otherwise $d=\lfloor \bar\gamma N/2\rfloor\ge 1$, so $d\le \bar\gamma N/2$ and hence $N/d\ge 2/\bar\gamma$.
Therefore
\[
m_0=\left\lceil \frac{N}{d}\right\rceil \ge \left\lceil \frac{2}{\bar\gamma}\right\rceil \ge \frac{1}{\bar\gamma},
\]
which implies $1/m_0\le \bar\gamma$. Thus $d$ is feasible and
$\SDA_{\mathbb Z}(Z,\bar\gamma)\ge \Omega(\bar\gamma N)$.

{\bf Upper bound.\enspace}
If $\bar\gamma\ge 1/2$, then $\bar\gamma N=\Theta(N)$ and trivially $\SDA_{\mathbb Z}(Z,\bar\gamma)\le N$,
so $\SDA_{\mathbb Z}(Z,\bar\gamma)=\Theta(\bar\gamma N)$.

Assume $\bar\gamma<1/2$ and let $d>4\bar\gamma N$. Then $N/d<1/(4\bar\gamma)$ and hence
\[
m_0=\left\lceil \frac{N}{d}\right\rceil \le \left\lceil \frac{1}{4\bar\gamma}\right\rceil < \frac{1}{\bar\gamma},
\]
where the last inequality uses $\bar\gamma<1/2$.
Therefore $1/m_0>\bar\gamma$, so the condition fails and no such $d$ is feasible.
Thus $\SDA_{\mathbb Z}(Z,\bar\gamma)\le O(\bar\gamma N)$ for $\bar\gamma<1/2$.

Combining the lower and upper bounds yields $\SDA_{\mathbb Z}(Z,\bar\gamma)=\Theta(\bar\gamma N)$ for $\bar\gamma\ge 1/N$.
\end{proof}

\begin{theorem}[Iteration/sample lower bound for $b$-bit mini-batch coordinate SGD on $k$-parity learning]
\label{thm:sgd-parity-lb}
Let $n\ge 1$ and $k\in\{1,\dots,n\}$, and let $N=\binom{n}{k}$.
For each $k$-sparse $s\in\mathcal S_k\subseteq\{0,1\}^n$, let $D_s$ be the realizable distribution over
$\mathcal X\times\{\pm1\}$ given by
\[
x\sim\Unif(\{0,1\}^n),\qquad y=f_s(x)=(-1)^{\langle s,x\rangle}.
\]
Consider a procedure that runs $T$ iterations of \emph{$b$-bit} coordinate mini-batch SGD with batch size $B$
(Definition~\ref{def:sgd-bbit}), and whose interaction with fresh examples $(x,y)\sim D_s$ is only through
the quantized per-example coordinate gradients (equivalently, through $1$-STAT$(b)$ access as in
Lemma~\ref{lem:sgd-is-1statb}). Let the procedure output a hypothesis $h:\{0,1\}^n\to\{\pm1\}$.

Assume that for every $s\in\mathcal S_k$ the procedure achieves nontrivial population error with probability
at least $\beta\ge 5/6$, namely
\[
\Pr  \left[\err_{D_s}(h)\le \frac14\right]\ \ge\ \beta,
\qquad\text{where}\qquad
\err_{D_s}(h):=\Pr_{(x,y)\sim D_s}[h(x)\neq y].
\]
Let $q:=TB$ be the total number of fresh examples used (equivalently, the number of $1$-STAT$(b)$ queries).
Then necessarily
\[
q\,2^{b}=\Omega\left(\sqrt{N}\right),
\]
and hence
\[
q=\Omega  \left(\frac{\sqrt{N}}{2^b}\right),
\qquad\text{and consequently}\qquad
T=\Omega  \left(\frac{\sqrt{N}}{B\,2^b}\right).
\]
\end{theorem}

\begin{proof}
{\bf Step 1: $b$-bit SGD is a $1$-STAT$(b)$ algorithm.\enspace}
By Lemma~\ref{lem:sgd-is-1statb}, a $T$-iteration run with batch size $B$ uses exactly
\[
q:=TB
\]
calls to $1$-STAT$(b)$.

{\bf Step 2: Simulate $1$-STAT$(b)$ by $\VSTAT$.\enspace}
Apply Theorem~B.4 in~\cite{pmlr-v65-feldman17c} with a fixed constant $\delta:=1/100$.
This yields an algorithm $\mathcal A'$ that succeeds with probability at least
$\beta-\delta> 2/3$ and uses at most
\[
Q = O(q\,2^b)
\]
queries to $\VSTAT(t)$ with
\[
t=\Theta  \left(\frac{q\,2^b}{\delta^2}\right)=\Theta(q\,2^b),
\]
where the hidden constants may depend on $\delta$ but $\delta$ is fixed.
In particular, by choosing $\delta$ sufficiently small as an absolute constant, we may assume that
whenever $q2^b\ge 1$ the corresponding parameter $t$ satisfies $t\ge 64$.

{\bf Step 3: Learning implies testing using one additional $\VSTAT(t)$ query.\enspace}
Define the promise testing problem $Z$: given oracle access to an unknown distribution $D$ that is
either $D_0$ (where $x\sim\Unif(\{0,1\}^n)$ and $y\sim\Unif(\{\pm1\})$ independent) or $D_s$ for some
unknown $s\in\mathcal S_k$, output $\mathrm{PLANTED}$ if $D=D_s$ and $\mathrm{NULL}$ if $D=D_0$.

Fix any $s\in\mathcal S_k$. Condition on the event that the learner outputs $h$ with
$\err_{D_s}(h)\le 1/4$, and consider the bounded query
\[
\phi_h(x,y) := \frac{1+h(x)y}{2}\in[0,1].
\]
Under $D_s$,
\[
\E_{(x,y)\sim D_s}[\phi_h(x,y)]
= \Pr[h(x)=y]
= 1-\err_{D_s}(h)
\ge \frac{3}{4}.
\]
Under $D_0$, since $y$ is independent uniform and $h(x)\in\{\pm1\}$,
\[
\E_{(x,y)\sim D_0}[h(x)y]=0
\qquad\Rightarrow\qquad
\E_{(x,y)\sim D_0}[\phi_h(x,y)] = \frac{1}{2}.
\]

Now make one additional $\VSTAT(t)$ query with $\phi_h$ and let $v$ be the oracle's response.
Output $\mathrm{PLANTED}$ iff $v\ge 5/8$, else output $\mathrm{NULL}$.

For the standard $\VSTAT(t)$ oracle \citep{doi:10.1137/16M1078471}, for any $[0,1]$-valued query
with mean $p$ the returned value satisfies
\[
|v-p|\le \max\Big\{\frac{1}{t},\,\sqrt{\frac{p(1-p)}{t}}\Big\}.
\]
Since $t\ge 64$, this error is at most $1/16$ both at $p=1/2$ and at $p=3/4$,
so the threshold $5/8$ separates the two cases:
\[
D=D_0:\ v \le \frac{1}{2}+\frac{1}{16}=\frac{9}{16}<\frac{5}{8},
\qquad
D=D_s:\ v \ge \frac{3}{4}-\frac{1}{16}=\frac{11}{16}>\frac{5}{8}.
\]
Thus, whenever $\err_{D_s}(h)\le 1/4$, this extra query yields a correct test.

Since $\mathcal A'$ outputs such an $h$ with probability at least $2/3$ on $D_s$,
the composed tester succeeds with probability at least $2/3$ as well.
It uses at most $Q+1=O(q2^b)$ queries to the \emph{same} oracle $\VSTAT(t)$.

{\bf Step 4: Apply the $\VSTAT$ lower bound.\enspace}
Apply Theorem~30 of \citet{reyzin2020statisticalqueriesstatisticalalgorithms} to the promise testing problem $Z$.
That theorem is parameterized by an accuracy $\bar\gamma>0$ and assumes oracle access to
$\VSTAT(\lceil c_1/\bar\gamma\rceil)$ for a universal constant $c_1>0$.

Fix another universal constant $c_2>0$ and set
\[
\bar\gamma := \frac{c_2}{t}.
\]
Then
\[
\left\lceil \frac{c_1}{\bar\gamma}\right\rceil
=
\left\lceil \frac{c_1}{c_2}\,t\right\rceil
=
\Theta(t).
\]
Since the $\VSTAT(t)$ accuracy guarantee is monotone in $t$ (larger $t$ means smaller allowed error), any response admissible for $\VSTAT(t)$ is also admissible for $\VSTAT(t')$ for every $t'\le t$. Therefore, by choosing the constant $c_2>0$ (in the definition $\bar\gamma=c_2/t$) sufficiently small so that
\[
\left\lceil \frac{c_1}{\bar\gamma}\right\rceil
=
\left\lceil \frac{c_1}{c_2}\,t\right\rceil
\le t,
\]
we ensure that oracle access to $\VSTAT(t)$ suffices to instantiate Theorem~30 at accuracy $\bar\gamma$ (up to constant factors in the parameterization).

Let $d:=\SDA(Z,\bar\gamma)$. By Lemma~\ref{lem:sda-parity},
\[
d=
\begin{cases}
\Theta(\bar\gamma N)=\Theta(N/t), & \text{if }\bar\gamma N\ge 1 \ \ (\text{i.e. } t\le c_2 N),\\
\Theta(1), & \text{if }\bar\gamma N<1 \ \ (\text{i.e. } t> c_2 N).
\end{cases}
\]
Theorem~30 implies that any algorithm that solves $Z$ with probability at least $2/3$
using only $\VSTAT(t)$ access must make at least $\Omega(d)$ oracle queries.

We have such a tester using $O(q2^b)$ queries. Therefore:
\begin{itemize}
\item If $t\le c_2 N$, then $d=\Theta(N/t)$ and hence
\[
q2^b=\Omega(N/t).
\]
With $t=\Theta(q2^b)$ from Step~2 this gives
\[
q2^b=\Omega  \left(\frac{N}{q2^b}\right)
\quad\Longrightarrow\quad
(q2^b)^2=\Omega(N)
\quad\Longrightarrow\quad
q=\Omega  \left(\frac{\sqrt N}{2^b}\right).
\]
\item If $t> c_2 N$, then $t=\Theta(q2^b)$ implies $q2^b=\Omega(N)$, which in particular implies
$q2^b=\Omega(\sqrt N)$ and hence $q=\Omega(\sqrt N/2^b)$ as well.
\end{itemize}
Finally, since $q=TB$, the lower bound on $q$ implies $T=\Omega  \left(\frac{\sqrt N}{B\,2^b}\right)$.
\end{proof}

\end{document}